\documentclass{article}

\usepackage[accepted]{icml2020}

\usepackage{microtype}
\usepackage{graphicx}
\usepackage{booktabs}
\usepackage[colorlinks,linkcolor=blue,citecolor=blue,urlcolor=blue]{hyperref}

\usepackage{bm}
\usepackage{amsmath}
\usepackage{amssymb}
\usepackage{caption}
\usepackage{subcaption}
\usepackage{multirow}
\newtheorem{theorem}{Theorem}
\newtheorem{definition}{Definition}
\newtheorem{lemma}{Lemma}
\newenvironment{proof}{{\noindent\it Proof.}\quad}{\hfill $\square$\par}

\begin{document}

\twocolumn[
\icmltitle{Unifying Graph Convolutional Neural Networks and Label Propagation}

\begin{icmlauthorlist}
	\icmlauthor{Hongwei Wang}{1}
	\icmlauthor{Jure Leskovec}{1}
\end{icmlauthorlist}

\icmlaffiliation{1}{Computer Science Department, Stanford University, Stanford, CA 94305, United States}
\icmlcorrespondingauthor{Hongwei Wang}{hongweiw@cs.stanford.edu}
\icmlkeywords{Graph neural networks; Label propagation algorithm}
\vskip 0.3in
]

\printAffiliationsAndNotice{}

\begin{abstract}
	Label Propagation (LPA) and Graph Convolutional Neural Networks (GCN) are both message passing algorithms on graphs.
	Both solve the task of node classification but LPA propagates node label information across the edges of the graph, while GCN propagates and transforms node feature information.
	However, while conceptually similar, theoretical relation between LPA and GCN has not yet been investigated.
	Here we study the relationship between LPA and GCN in terms of two aspects: (1) \textit{feature/label smoothing} where we analyze how the feature/label of one node is spread over its neighbors; And, (2) \textit{feature/label influence} of how much the initial feature/label of one node influences the final feature/label of another node.
	Based on our theoretical analysis, we propose an end-to-end model that unifies GCN and LPA for node classification.
	In our unified model, edge weights are learnable, and the LPA serves as regularization to assist the GCN in learning proper edge weights that lead to improved classification performance.
	Our model can also be seen as learning attention weights based on node labels, which is more task-oriented than existing feature-based attention models.
	In a number of experiments on real-world graphs, our model shows superiority over state-of-the-art GCN-based methods in terms of node classification accuracy.
\end{abstract}

\section{Introduction}
	Consider the problem of node classification in a graph, where the goal is to learn a mapping $\mathcal M: \mathcal V \rightarrow \mathcal L$ from node set $\mathcal V$ to label set $\mathcal L$.
	Solution to this problem is widely applicable to various scenarios, e.g., inferring income of users in a social network or classifying scientific articles in a citation network.
	Different from a generic machine learning problem where samples are independent from each other, nodes are connected by edges in the graph, which provide additional information and require more delicate modeling.
	To capture the graph information, researchers have mainly designed models on the assumption that labels and features vary \textit{smoothly} over the edges of the graph. 
	In particular, on the label side $\mathcal L$, node labels are propagated and aggregated along edges in the graph, which is known as \textit{Label Propagation Algorithm} (LPA) \citep{zhu2005semi, zhou2004learning, zhang2007hyperparameter, wang2008label, karasuyama2013manifold, gong2017label, liu2018learning};
	On the node side $\mathcal V$, node features are propagated along edges and transformed through neural network layers, which is known as \textit{Graph Convolutional Neural Networks} (GCN) \citep{kipf2017semi, hamilton2017inductive, li2018deeper, xu2018representation,liao2019lanczosnet,xu2019powerful,qu2019gmnn}.
	
	GCN and LPA are related in that they propagate features and labels on the two sides of the mapping $\mathcal M$, respectively.
	However, the relationship between GCN and LPA has not yet been investigated.
	Specifically, what is the theoretical relationship between GCN and LPA, and how can they be combined to develop a more accurate model for node classification in graphs?
	
	Here we study the theoretical relationship between GCN and LPA from two viewpoints:
	(1) \textit{Feature/label smoothing}, where we show that the intuition behind GCN/LPA is smoothing features/labels of nodes across the edges of the graph, i.e., one node's feature/label equals the weighted average of features/labels of its neighbors.
	We prove that if the weights of edges in a graph smooth the node features with high precision, they also smooth the node labels with guaranteed upper bound on the smoothing error.
	And, (2) \textit{feature/label influence},
	where we quantify how much the initial feature/label of node $v_b$ influences the output feature/label of node $v_a$ in GCN/LPA by studying the Jacobian/gradient of node $v_b$ with respect to node $v_a$.
	We also prove the quantitative relationship between feature influence and label influence.
	
	Based on the above theoretical analysis, we propose a unified model GCN-LPA for node classification.
	We show that the key to improving the performance of GCN is to enable nodes within the same class/label to connect more strongly with each other by making edge weights/strengths trainable.
	Then we prove that increasing the strength of edges between the nodes of the same class is equivalent to increasing the accuracy of LPA's predictions.
	Therefore, we can first learn the optimal edge weights by minimizing the loss of predictions in LPA, then plug the optimal edge weights into a GCN to learn node representations and do final classification.
	In GCN-LPA, we further combine the two steps together and train the whole model in an end-to-end fashion, where the LPA part serves as regularization to assist the GCN part in learning proper edge weights that benefit the separation of different node classes.
	It is worth noticing that GCN-LPA can also be seen as learning \textit{attention weights} for edges based on \textit{node label information}, which requires less handcrafting and is more task-oriented than existing work that learns attention weights based on \textit{node feature similarity} \citep{velivckovic2018graph,thekumparampil2018attention,zhang2018gaan,liu2019geniepath}.
	
	We conduct extensive experiments on five datasets, and the results indicate that our model outperforms state-of-the-art methods in terms of classification accuracy.
	The experimental results also show that combining GCN and LPA together is able to learn more informative edge weights thereby leading to better performance.

\section{Unifying GCN and LPA}
    In this section, we first formulate the node classification problem and briefly introduce LPA and GCN.
    We then prove their relationship from the viewpoints of smoothing and influence.
    Based on the theoretical findings, we propose a unified model GCN-LPA, and analyze why our model is theoretically superior to vanilla GCN.
    
	\subsection{Problem Formulation and Preliminaries}
		We begin by describing the problem of node classification on graphs and introducing notation.
		Consider a graph $\mathcal G = (\mathcal V, A, X, Y)$, where $\mathcal V = \{v_1,\cdots, v_n\}$ is the set of nodes, $A \in \mathbb R^{n \times n}$ is the adjacency matrix (self-loops are included), $X$ is the feature matrix of nodes and $Y$ is labels of nodes.
		$a_{ij}$ (the $ij$-th entry of $A$) is the weight of the edge connecting $v_i$ and $v_j$.
		$\mathcal N(v)$ denotes the set of immediate neighbors of node $v$ in graph $\mathcal G$.
		Each node $v_i$ has a feature vector ${\bf x}_i$ which is the $i$-th row of $X$, while only the first $m$ nodes have labels $y_1,\cdots, y_m$ from a label set $\mathcal L = \{1,\cdots, c\}$.
		The goal is to learn a mapping $\mathcal M: \mathcal V \rightarrow \mathcal L$ and predict labels of unlabeled nodes.
	
		\textbf{Label Propagation Algorithm}.
		LPA assumes that two connected nodes are likely to have the same label, and thus it propagates labels iteratively along the edges.
		Let $ Y^{(k)} = [y_1^{(k)},\cdots,  y_n^{(k)}]^\top \in \mathbb R^{n \times c}$ be the soft label matrix in iteration $k>0$, in which the $i$-th row $y_i^{{(k)}\top}$ denotes the predicted label distribution for node $v_i$ in iteration $k$.
		When $k=0$, the initial label matrix $Y^{(0)} = [y_1^{(0)},\cdots,  y_n^{(0)}]^\top$ consists of one-hot label indicator vectors $y_i^{(0)}$ for $i = 1,\cdots, m$ (i.e., labeled nodes) or zero vectors otherwise (i.e., unlabeled nodes).
		Let $D$ be the diagonal degree matrix for $A$ with entries $d_{ii} = \sum_j a_{ij}$.
		Then LPA \citep{zhu2005semi} in iteration $k$ is formulated as the following two steps:
		\begin{eqnarray}
			Y^{(k+1)} = D^{-1} A \ Y^{(k)}, \label{eq:lp}\\
			y_i^{(k+1)} = y_i^{(0)}, \ \forall \ i \le m. \label{eq:lp_2}
		\end{eqnarray}
		In Eq. (\ref{eq:lp}), all nodes propagate labels to their neighbors according to normalized edge weights.
		Then in Eq. (\ref{eq:lp_2}), labels of all labeled nodes are reset to their initial values, because LPA wants to persist labels of nodes which are labeled so that unlabeled nodes do not overpower the labeled ones as the initial labels would otherwise fade away.
	
		\textbf{Graph Convolutional Neural Network}.
		GCN is a multi-layer feedforward neural network that propagates and transforms node features across the graph.
		The layer-wise propagation rule of GCN is $X^{(k+1)} = \sigma (D^{-\frac{1}{2}} A D^{-\frac{1}{2}} X^{(k)} W^{(k)})$, where $W^{(k)}$ is trainable weight matrix in the $k$-th layer, $\sigma(\cdot)$ is an activation function such as $\text{ReLU}$, and $X^{(k)} = [{\bf x}_1^{(k)},\cdots, {\bf x}_n^{(k)}]^\top$ are the $k$-th layer node representations with $X^{(0)} = X$.
		To align with the above LPA, we use $D^{-1}A$ as the normalized adjacency matrix instead of the symmetric one $D^{-\frac{1}{2}} A D^{-\frac{1}{2}}$ proposed by \cite{kipf2017semi}.
		Therefore, the feature propagation scheme of GCN in layer $k$ is:
		\begin{equation}
		\label{eq:gcn}
			X^{(k+1)} = \sigma \left( D^{-1} A X^{(k)} W^{(k)} \right).
		\end{equation}
		Notice similarity between Eqs. (\ref{eq:lp}) and (\ref{eq:gcn}). Next we shall study and uncover the relationship between the two equations.

	\subsection{Feature Smoothing and Label Smoothing}
		The intuition behind both LPA and GCN is \textit{smoothing} \citep{zhu2003semi,li2018deeper}:
		In LPA, the final label of a node is the weighted average of labels of its neighbors:
		\begin{equation}
			y_i^{(\infty)} = \frac{1}{d_{ii}} \sum_{j \in \mathcal N(i)} a_{ij} y_j^{(\infty)}.
		\end{equation}
		In GCN, the final node representation is also the weighted average of representations of its neighbors if we assume $\sigma$ is identity function and $W^{(\cdot)}$ are identity matrices:
		\begin{equation}
			{\bf x}_i^{(\infty)} = \frac{1}{d_{ii}} \sum_{j \in \mathcal N(i)} a_{ij} {\bf x}_j^{(\infty)}.
		\end{equation}
		Next we show the relationship between feature smoothing and label smoothing:
		
		\begin{theorem}
		\label{thm:smoothing}
			\rm\textbf{(Relationship between feature smoothing and label smoothing)}
			Suppose that the latent ground-truth mapping $\mathcal M: {\bf x} \rightarrow y$ from node features to node labels is differentiable and satisfies $L$-Lipschitz constraint, i.e., $| \mathcal M({\bf x}_1) - \mathcal M({\bf x}_2) | \leq L \| {\bf x}_1 - {\bf x}_2 \|_2$ for any ${\bf x}_1$ and ${\bf x}_2$ ($L$ is a constant).
			If the edge weights $\{a_{ij}\}$ approximately smooth ${\bf x}_i$ over its immediate neighbors with error $\epsilon_i$, i.e.,
			\begin{equation}
				{\bf x}_i = \frac{1}{d_{ii}} \sum_{j \in \mathcal N(i)} a_{ij} {\bf x}_j + \epsilon_i,
			\end{equation}
			then the edge weights $\{a_{ij}\}$ also approximately smooth $y_i$ over its immediate neighbors with the following approximation error:
			\begin{equation}
				\big| y_i - \frac{1}{d_{ii}} \sum_{j \in \mathcal N(i)} a_{ij} y_j \big| \leq L \| \epsilon_i\|_2 + o \big( \max_{j \in \mathcal N(i)} ( \| {\bf x}_j - {\bf x}_i \|_2 ) \big),
			\end{equation}
			where $o(\alpha)$ denotes a higher order infinitesimal than $\alpha$.
		\end{theorem}
		
		Proof of Theorem \ref{thm:smoothing} is in Appendix \ref{app:a}.
		Theorem \ref{thm:smoothing} indicates that label smoothing is theoretically guaranteed by feature smoothing.
		Note that if we treat edge weights $\{a_{ij}\}$ learnable, then feature smoothing (i.e., $\epsilon_i \rightarrow 0$) can be directly achieved by keeping node features ${\bf x}_i$ fixed while setting $\{a_{ij}\}$ appropriately, without resorting to feature propagation in a multi-layer GCN.
		Therefore, a simple approach to exploit this theorem would be to learn $\{a_{ij}\}$ by reconstructing node feature ${\bf x}_i$ from its neighbors, then use the learned $\{a_{ij}\}$ to reconstruct node labels $y_i$ \citep{karasuyama2013manifold}.
		
		As shown in Theorem \ref{thm:smoothing}, the approximation error of labels is dominated by $L \| \epsilon_i\|_2$.
		However, this error could be fairly large in practice because:
		(1) The number of immediate neighbors for a given node may be too small to reconstruct its features perfectly, especially in the case where node features are high-dimensional and sparse.
		For example, in a citation network where node features are one-hot bag-of-words vectors, the feature of one article can never be precisely reconstructed if none of its neighboring articles contains the specific word that appears in this article.
		As a result, $\| \epsilon_i\|_2$ will be non-neglibible.
		This explains why it is beneficial to apply LPA and GCN for multiple iterations/layers in order to include information from farther away neighbors.
		(2) The ground-truth mapping $\mathcal M$ may not be sufficiently smooth due to the complex structure of latent manifold and possible noise, which fails to satisfy $L$-Lipschitz constraint.
		In other words, the constant $L$ will be extremely large.
		
	\subsection{Feature Influence and Label Influence}
	    To address the above concerns and extend our analysis, we next consider GCN and LPA with multiple layers/iterations, and do not impose any constraint on the ground-truth mapping $\mathcal M$.
		
		Consider two nodes $v_a$ and $v_b$ in a graph.
		Inspired by \cite{koh2017understanding} and \cite{xu2018representation}, we study the relationship between GCN and LPA in terms of influence, i.e., how the output feature/label of $v_a$ will change if the initial feature/label of $v_b$ is varied slightly.
		Technically, the feature/label influence is measured by the Jacobian/gradient of the output feature/label of $v_a$ with respect to the initial feature/label of $v_b$.
		Denote ${\bf x}_a^{(k)}$ as the $k$-th layer representation vector of $v_a$ in GCN, and ${\bf x}_b$ as the initial feature vector of $v_b$.
		We quantify the feature influence of $v_b$ on $v_a$ as follows:
		
		\begin{definition}
			\rm\textbf{(Feature influence)}
			The feature influence of node $v_b$ on node $v_a$ after $k$ layers of GCN is the L1-norm of the expected Jacobian matrix $\partial {\bf x}_a^{(k)} / \partial {\bf x}_b$:
			\begin{equation}
				I_f(v_a, v_b; k) = \big\| \mathbb E \big[ \partial {\bf x}_a^{(k)} / \partial {\bf x}_b \big] \big\|_1.
			\end{equation}
			The normalized feature influence is then defined as
			\begin{equation}
				\tilde I_f(v_a, v_b; k) = \frac{I_f(v_a, v_b; k)}{\sum\nolimits_{v_i \in \mathcal V} I_f(v_a, v_i; k)}.
			\end{equation}
		\end{definition}
		
		We also consider the label influence of node $v_b$ on node $v_a$ in LPA (this implies that $v_a$ is unlabeled and $v_b$ is labeled).
		Since different label dimensions of $y_i^{(\cdot)}$ do not interact with each other in LPA, we assume that all $y_i$ and $y_i^{(\cdot)}$ are scalars within $[0, 1]$ (i.e., a binary classification) for simplicity.
		Label influence is defined as follows:
		
		\begin{definition}
			\rm\textbf{(Label influence)}
			The label influence of labeled node $v_b$ on unlabeled node $v_a$ after $k$ iterations of LPA is the gradient of $y_a^{(k)}$ with respect to $y_b$:
			\begin{equation}
				I_l(v_a, v_b; k) = \partial y_a^{(k)} / \partial y_b.
			\end{equation}
			
		\end{definition}
		
		The following theorem shows the relationship between feature influence and label influence:
		
		\begin{theorem}
		\label{thm:influence}
			\rm\textbf{(Relationship between feature influence and label influence)}
			Assume the activation function used in GCN is ReLU.
			Denote $v_a$ as an unlabeled node, $v_b$ as a labeled node, and $\beta$ as the fraction of unlabeled nodes.
			Then the label influence of $v_b$ on $v_a$ after $k$ iterations of LPA equals, in expectation, to the cumulative normalized feature influence of $v_b$ on $v_a$ after $k$ layers of GCN:
			\begin{equation}
				\mathbb E \big[ I_l(v_a, v_b; k) \big] = \sum\nolimits_{j=1}^k \beta^j \tilde I_f(v_a, v_b; j).
			\end{equation}
		\end{theorem}
		
		Proof of Theorem \ref{thm:influence} is in Appendix \ref{app:b}.
		Intuitively, Theorem \ref{thm:influence} shows that if $v_b$ has high label influence on $v_a$, then the initial feature vector of $v_b$ will also affect the output feature vector of $v_a$ to a large extent.
		Theorem \ref{thm:influence} provides the theoretical guideline for designing our unified model in the next subsection.

	\subsection{The Unified Model}
	\label{sec:model}
		Before introducing the proposed model, we first rethink the GCN method and see what an ideal node representation should be like.
		Since we aim to classify nodes, the perfect node representation would be such that nodes with the same label are embedded close together, which would give a large separation between different classes.
		Intuitively, the key to achieve this goal is to enable nodes within the same class to connect more strongly with each other, so that they are pushed together by the GCN.
		We can therefore make edge strengths/weights trainable, then learn to increase the \textit{intra-class feature influence} for each class $i$:
		\begin{equation}
			\label{eq:intra_class_feature_influence}
			\sum_{v_a, v_b: y_a = i, y_b = i} \tilde I_f(v_a, v_b)
		\end{equation}
		by adjusting edge weights.
		However, this requires operating on Jacobian matrices with the size of $d^{(0)} \times d^{(K)}$ ($d^{(0)}$ and $d^{(K)}$ are the dimensions of initial and output features, respectively), which is impractical if initial node features are high-dimensional.
		Fortunately, we can turn to optimizing the \textit{intra-class label influence} instead of Eq. (\ref{eq:intra_class_feature_influence}), i.e.,
		\begin{equation}
			\label{eq:intra_class_label_influence}
			\sum_{v_a, v_b: y_a = i, y_b = i} I_l(v_a, v_b),
		\end{equation}
		according to Theorem \ref{thm:influence}.
		We further show that, by the following theorem, the total intra-class label influence on a given node $v_a$ is proportional to the probability that $v_a$ is classified correctly by LPA:
		\begin{theorem}
		\label{thm:lpa}
			\rm\textbf{(Relationship between label influence and LPA's prediction)}
			Consider a given node $v_a$ and its label $y_a$.
			If we treat node $v_a$ as unlabeled, then the total label influence of nodes with label $y_a$ on node $v_a$ is proportional to the probability that node $v_a$ is classified as $y_a$ by LPA:
			\begin{equation}
				\sum_{v_b: y_b = y_a} I_l(v_a, v_b; k) \propto \Pr \big( \hat y_a^{lpa} = y_a \big),
			\end{equation}
			where $\hat y_a^{lpa}$ is the predicted label of $v_a$ using a $k$-iteration LPA.
		\end{theorem}
		
		\begin{figure}
			\centering
			\captionsetup[subfigure]{justification=centering}
			\begin{subfigure}[b]{0.22\textwidth}
   				\includegraphics[width=\textwidth]{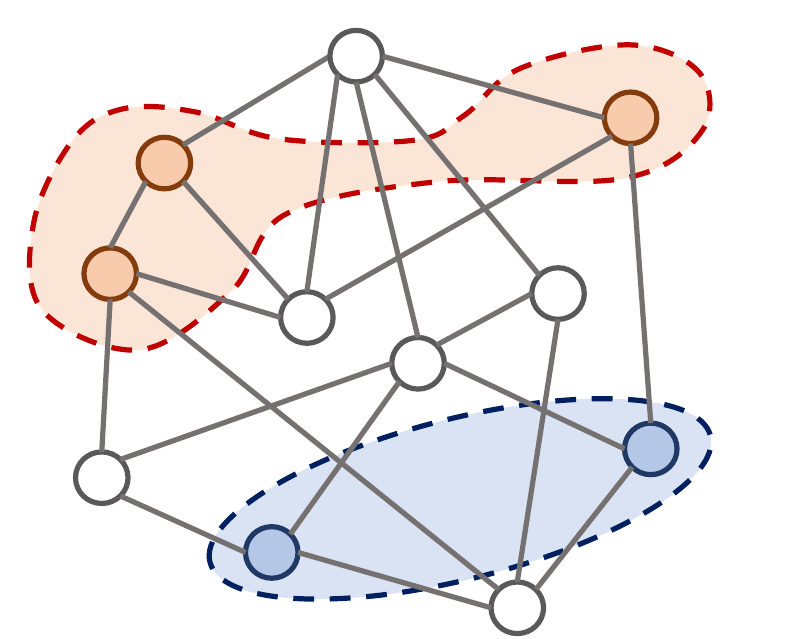}
   				\caption{A graph with two classes of nodes (red vs. blue)}
   				\label{fig:model_1}
			\end{subfigure}
			\hfill
			\begin{subfigure}[b]{0.22\textwidth}
				\includegraphics[width=\textwidth]{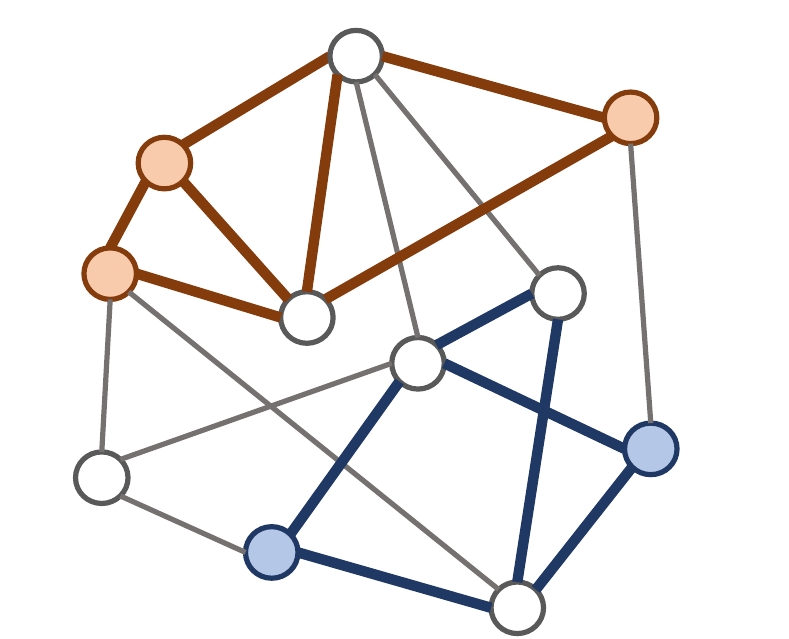}
				\caption{Potential intra-class edges (bold links)}
				\label{fig:model_2}
			\end{subfigure}
			\caption{A graph with two classes of nodes, while white nodes are unlabeled (Figure \ref{fig:model_1}). To ease the separation of the two classes, our model will increase the connecting strength among nodes within the same class (i.e., within one dotted circle), thereby increasing their feature/label influence on each other. In this way, our model is able to identify potential intra-class edges (bold links in Figure \ref{fig:model_2}) and strengthen their weights.}
			\label{fig:model}
		\end{figure}
		
		Proof of Theorem \ref{thm:lpa} is in Appendix \ref{app:c}.
		Theorem \ref{thm:lpa} indicates that, if edge weights $\{a_{ij}\}$ maximize the probability that $v_a$ is correctly classified by LPA, then they also maximize the intra-class label influence on node $v_a$.
		We can therefore first learn the optimal edge weights $A^*$ by minimizing the loss of predicted labels by LPA:\footnote{Here the optimal edge weights $A^*$ share the same topology as the original graph $\mathcal G$, meaning that we do not add or remove edges from $\mathcal G$ but only learning the weights of existing edges. See the end of this subsection for more discussion.}
		\begin{equation}
		\begin{split}
			A^* =& \mathop{\arg\min}_A L_{lpa}(A)\\
			=& \mathop{\arg\min}_A \ \frac{1}{m} \sum_{v_a: a \leq m} J(\hat y_a^{lpa}, y_a),
		\end{split}
		\end{equation}
		where $J$ is the cross-entropy loss, $\hat y_a^{lpa}$ and $y_a$ are the predicted label distribution of $v_a$ using LPA and the true one-hot label of $v_a$, respectively.\footnote{Here we somewhat abuse the notations for simplicity, since in Theorem \ref{thm:lpa} the two notations represent label category rather than label distribution. But the subtle difference can be easily distinguished based on context.}
		$a \leq m$ means $v_a$ is labeled.
		The optimal $A^*$ maximize the probability that each node is correctly labeled by LPA, thus also increasing the intra-class label influence (by Theorem \ref{thm:lpa}) and intra-class feature influence (by Theorem \ref{thm:influence}).
		Then we can apply $A^*$ and the corresponding $D^*$ to a GCN to predict labels:
		\begin{equation}
			\label{eq:k_gcn}
			X^{(k+1)} = \sigma ({D^*}^{-1} A^* X^{(k)} W^{(k)}), \ k = 0, 1,\cdots, K-1.
		\end{equation}
		We use $\hat y_a^{gcn}$, the $a$-th row of $X^{(K)}$, to denote the predicted label distribution of $v_a$ using the GCN specified in Eq. (\ref{eq:k_gcn}).
		The the optimal transformation matrices in the GCN can be learned by minimizing the loss of predicted labels by GCN:
		\begin{equation}
		\begin{split}
			W^* =& \mathop{\arg\min}_W L_{gcn}(W, A^*)\\
			=& \mathop{\arg\min}_W \frac{1}{m} \sum_{v_a: a \leq m} J(\hat y_a^{gcn}, y_a),
		\end{split}
		\end{equation}
		
		\begin{figure*}
			\centering
			\captionsetup[subfigure]{justification=centering}
			\begin{subfigure}[b]{0.215\textwidth}
   				\includegraphics[width=\textwidth]{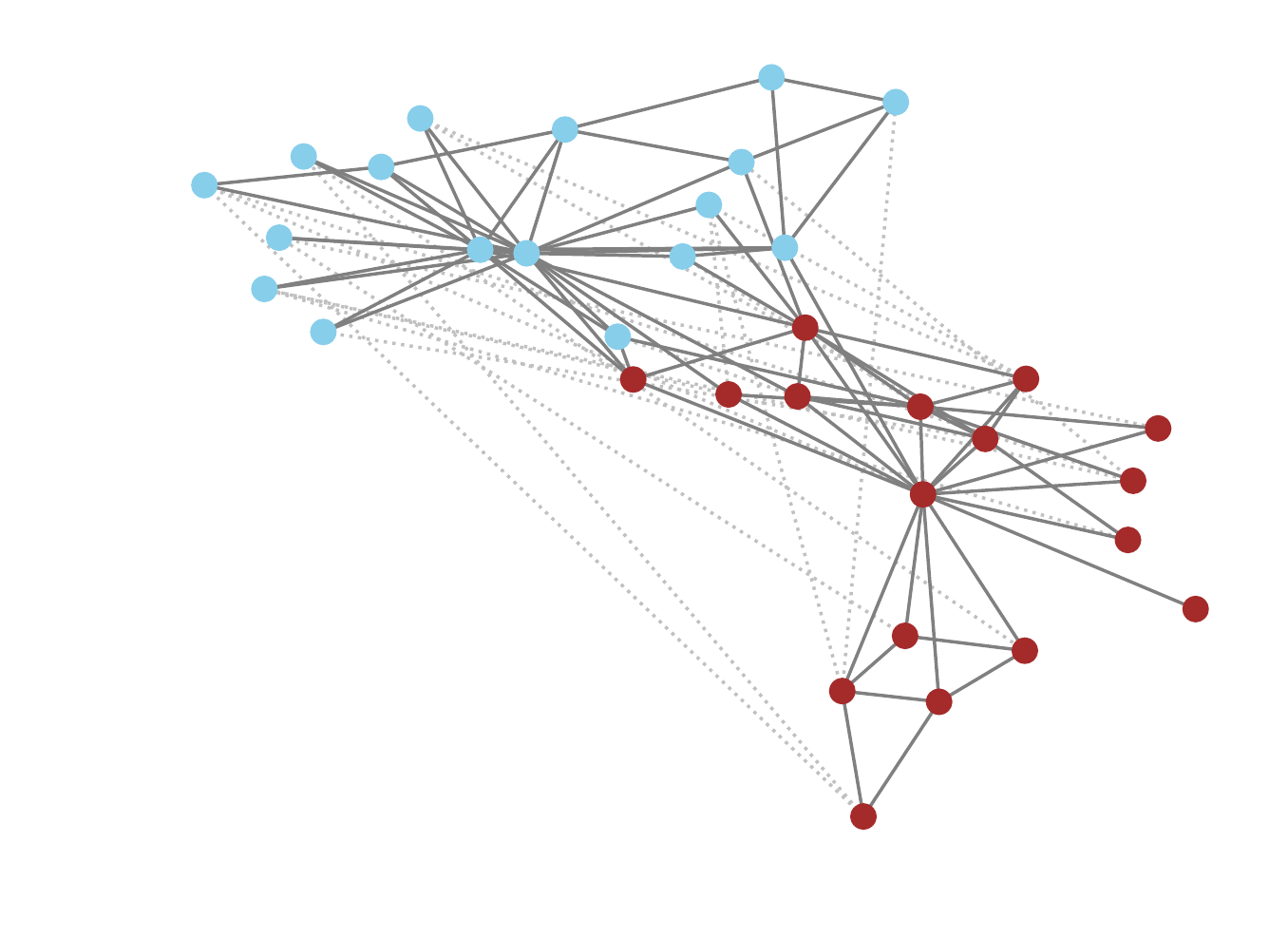}
   				\caption{Karate club network\\with noisy edges}
   				\label{fig:karate_1}
			\end{subfigure}
			\hfill
			\begin{subfigure}[b]{0.19\textwidth}
   				\includegraphics[width=\textwidth]{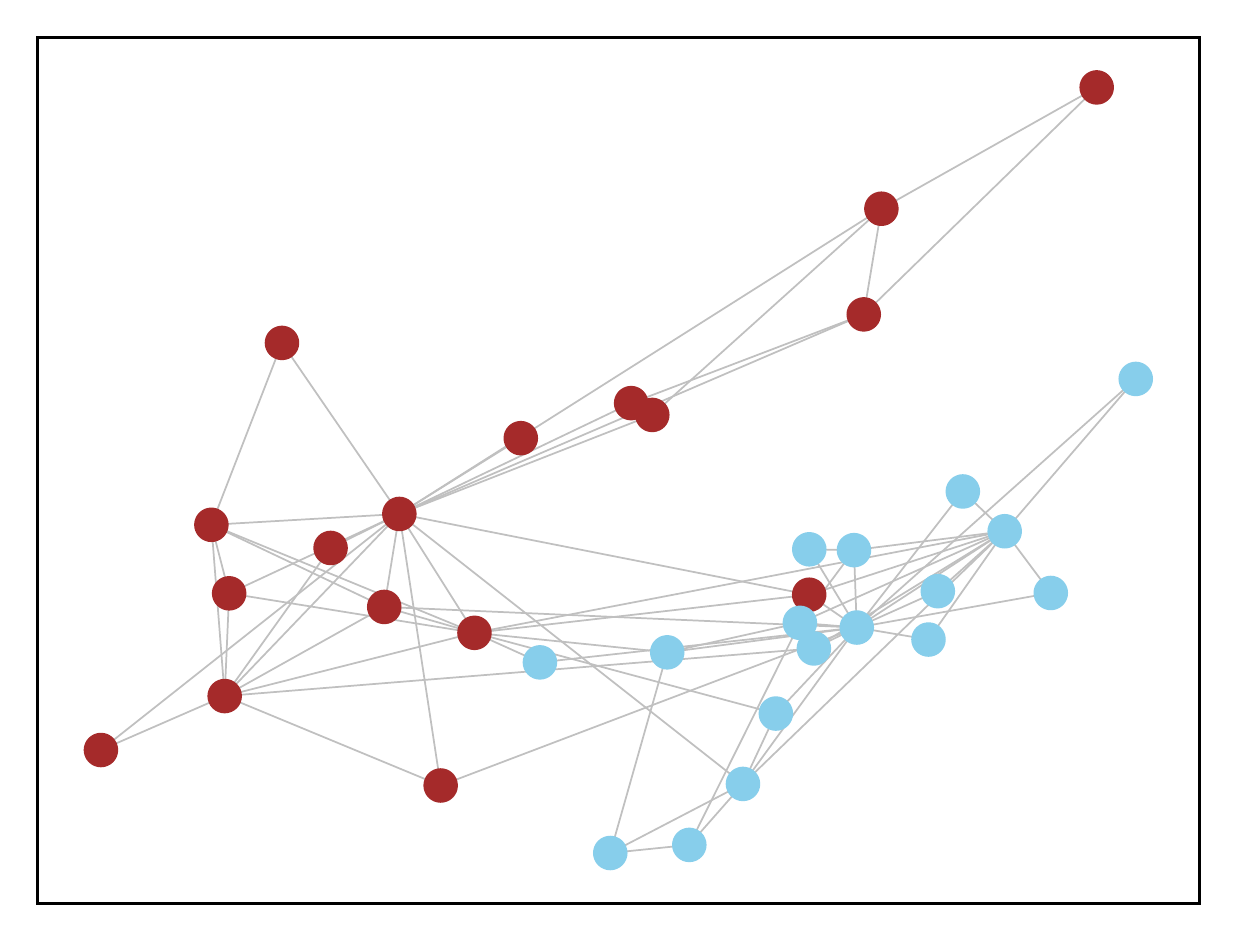}
   				\caption{GCN on the\\original network}
   				\label{fig:karate_2}
			\end{subfigure}
			\hfill
			\begin{subfigure}[b]{0.19\textwidth}
				\includegraphics[width=\textwidth]{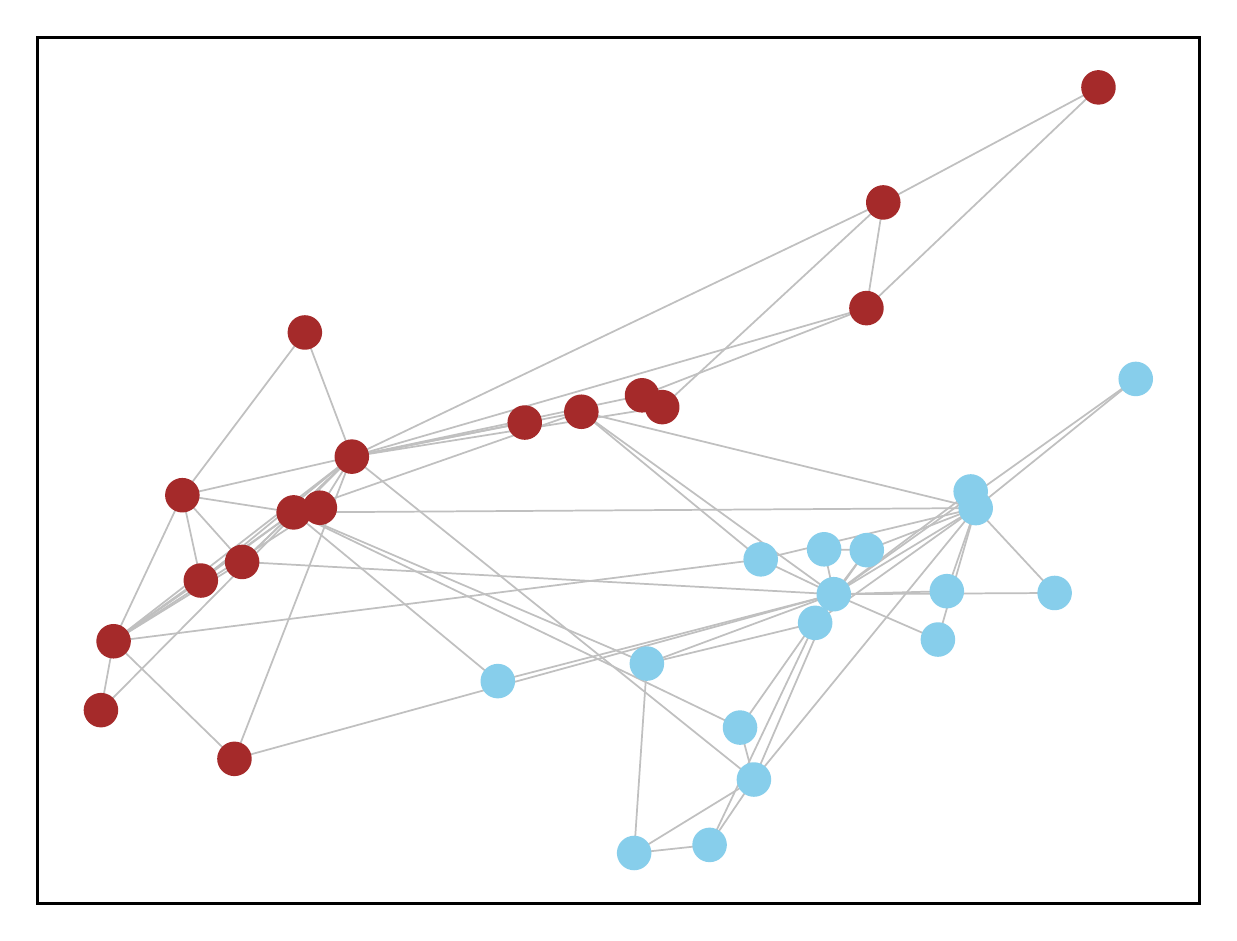}
				\caption{GCN-LPA on the\\original network}
				\label{fig:karate_3}
			\end{subfigure}
			\hfill
			\begin{subfigure}[b]{0.19\textwidth}
   				\includegraphics[width=\textwidth]{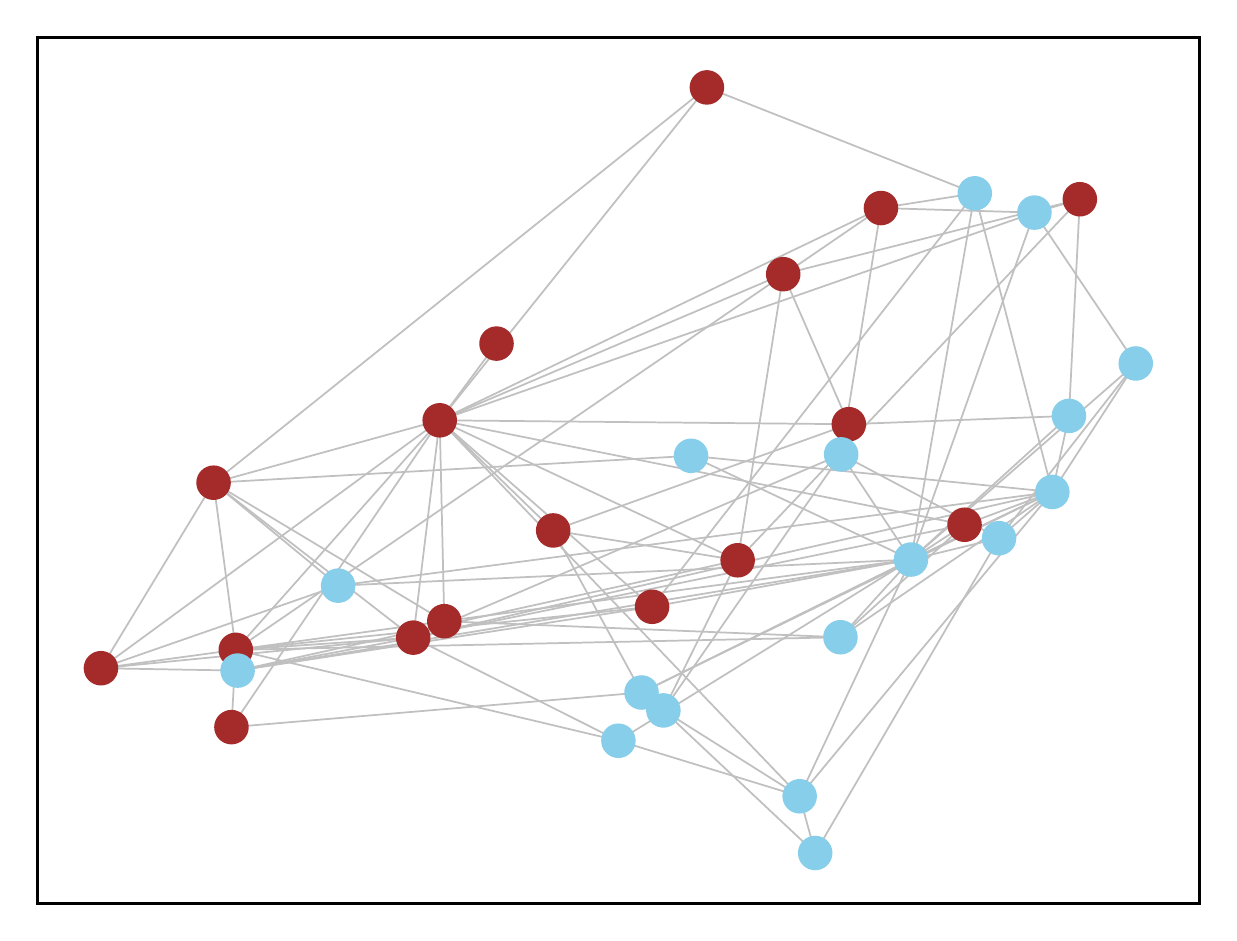}
   				\caption{GCN on the\\noisy network}
   				\label{fig:karate_4}
			\end{subfigure}
			\hfill
			\begin{subfigure}[b]{0.19\textwidth}
				\includegraphics[width=\textwidth]{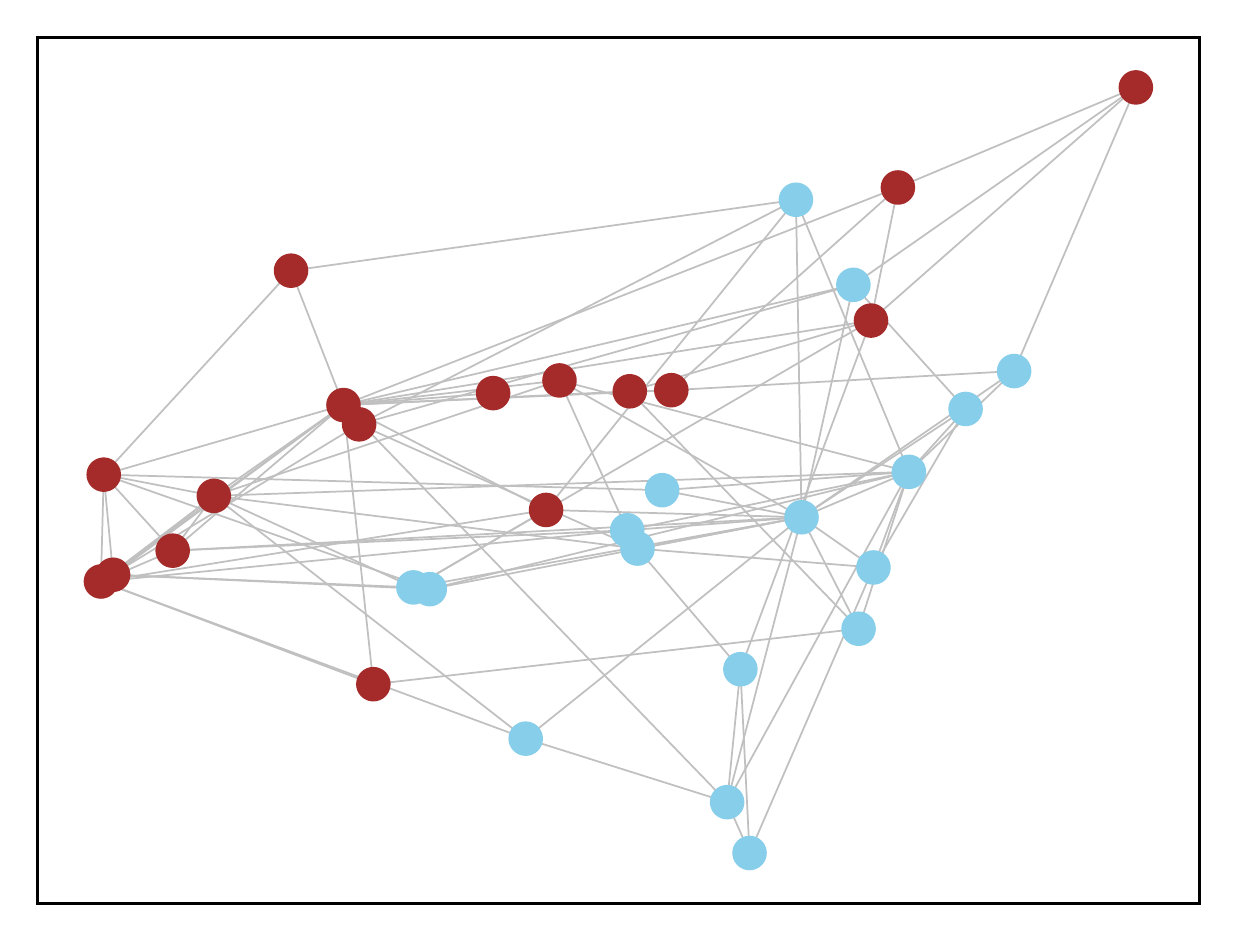}
				\caption{GCN-LPA on the\\noisy network}
				\label{fig:karate_5}
			\end{subfigure}
			\caption{Node embeddings of Zachary's karate club network trained on a node classification task (red vs. blue). Figure \ref{fig:karate_1} visualizes the graph. Node coordinates in Figure \ref{fig:karate_2}-\ref{fig:karate_5} are the embedding coordinates. Notice that GCN does not produce linearly separable embeddings (Figure \ref{fig:karate_2} vs. Figure \ref{fig:karate_3}), while GCN-LPA performs much better even in the presence of noisy edges (Figure \ref{fig:karate_4} vs. Figure \ref{fig:karate_5}).
			Additional visualizations are included in Appendix \ref{app:e}.}
			\label{fig:karate}
		\end{figure*}
		
		In practice, it is generally better to combine the above two steps together and train the whole model in an end-to-end fashion:
		\begin{equation}
			W^*, A^* = \mathop{\arg\min}_{W, A} \ L_{gcn}(W, A) + \lambda L_{lpa}(A),
		\end{equation}
		where $\lambda$ is the balancing hyper-parameter.
		In this way, $L_{lpa}(A)$ serves as a regularization term that assists the learning of edge weights $A$, since it is hard for the GCN to learn both $W$ and $A$ simultaneously due to overfitting.
		The proposed GCN-LPA approach can also be seen as learning the importance of edges that can be used to reconstruct node labels accurately by LPA, then transferring this knowledge from label space to feature space for the GCN.
		From this perspective, GCN-LPA also connects to Theorem \ref{thm:smoothing} except that the knowledge transfer is in the other direction.
		
		It is also worth noticing how the optimal $A^*$ is configured.
		The principle here is that we do not modify the basic structure of the original graph (i.e., not adding or removing edges) but only adjusting weights of existing edges.
		This is equivalent to learning a positive mask matrix $M$ for the adjacency matrix $A$ and taking the Hadamard product $M \circ A = A^*$.
		Each element $M_{ij}$ can be set as either a free variable or a function of the nodes at edge endpoints, for example, $M_{ij} = \log \left( \exp({\bf x}_i^\top {\bf H} {\bf x}_j) + 1 \right)$ where $\bf H$ is a learnable kernel matrix for measuring feature similarity.

	\subsection{Analysis of GCN-LPA Model Behavior}
	\label{sec:discussion}
		In this subsection, we show benefits of our unified model compared with GCN by analyzing properties of embeddings produced by the two models.
    	We first analyze the update rule of GCN for node $v_i$:
    	\begin{equation}
    		{\bf x}_i^{(k+1)} = \sigma \left( \sum_{v_j \in \mathcal N(v_i)} \tilde a_{ij} {\bf x}_j^{(k)} W^{(k)} \right), 
    	\end{equation}
    	where $\tilde a_{ij} = a_{ij} / d_{ii}$ is the normalized weight of edge $(j, i)$.
		This formula can be decomposed into the following two steps:
		
		(1) In \textit{aggregation} step, we calculate the aggregated representation ${\bf h}_i^{(k)}$ of all neighborhoods $\mathcal N(v_i)$:
		\begin{equation}
			{\bf h}_i^{(k)} = \sum_{v_j \in \mathcal N(v_i)} \tilde a_{ij} {\bf x}_j^{(k)}.
		\end{equation}
		(2) In \textit{transformation} step, the aggregated representation ${\bf h}_i^{(k)}$ is mapped to a new space by a transformation matrix and nonlinear function:
		\begin{equation}
			{\bf x}_i^{(k+1)} = \sigma \big( {\bf h}_i^{(k)} W^{(k)} \big).
		\end{equation}
		
		We show by the following theorem that the aggregation step reduces the overall distance in the embedding space between the nodes that are connected in the graph:
		\begin{theorem}
		\label{thm:decrease}
			\rm\textbf{(Shrinking property in GCN)}
			Let $D({\bf x}) = \frac{1}{2} \sum_{v_i, v_j} \tilde a_{ij} \| {\bf x}_i - {\bf x}_j \|_2^2$ be a distance metric over node embeddings ${\bf x}$.
			Then we have
			\begin{equation*}
				D({\bf h}^{(k)}) \leq D({\bf x}^{(k)}).
			\end{equation*}
		\end{theorem}
		
		\begin{table*}[t]
			\centering
			\setlength{\tabcolsep}{12pt}
			\begin{tabular}{c|ccccc}
				\hline
				& \textbf{Cora} & \textbf{Citeseer} & \textbf{Pubmed} & \textbf{Coauthor-CS} & \textbf{Coauthor-Phy} \\
				\hline
				\# nodes & 2,708 & 3,327 & 19,717 & 18,333 & 34,493 \\
				\# edges & 5,278 & 4,552 & 44,324 & 81,894 & 247,962 \\
				\# features & 1,433 & 3,703 & 500 & 6,805 & 8,415 \\
				\# classes & 7 & 6 & 3 & 15 & 5 \\
				Intra-class edge rate & 81.0\% & 73.6\% & 80.2\% & 80.8\% & 93.1\% \\
				\hline
			\end{tabular}
			\caption{Dataset statistics after removing self-loops and duplicate edges.}
			\label{table:statistics}
		\end{table*}
	
		Proof of Theorem \ref{thm:decrease} is in Appendix \ref{app:d}.
		Theorem \ref{thm:decrease} indicates that the overall distance among connected nodes is reduced after taking one aggregation step, which implies that connected components in the graph ``shrink'' and nodes within each connected component get closer to each other in the embedding space.
		In an ideal case where edges only connect nodes with the same label, the aggregation step will push nodes within the same class together, which greatly benefits the transformation step that acts like a hyperplane $W^{(k)}$ for classification.
		However, two connected nodes may have different labels.
		These ``noisy'' edges will impede the formation of clusters and make the inter-class boundary less clear.
	
		Fortunately, in GCN-LPA, edge weights are learned by minimizing the difference between ground-truth labels and labels reconstructed from multi-hop neighbors.
		This will force the model to increase weight/bandwidth of possible paths that connect nodes with the same label, so that labels can ``flow'' easily along these paths for the purpose of label reconstruction.
		In this way, GCN-LPA is able to identify potential intra-class edges and increase their weights to assist learning clustering structures.
		Figure \ref{fig:model} gives a toy example illustrating how our model works intuitively.
		
		To empirically justify our claim, we apply a two-layer untrained GCN with randomly initialized transformation matrices to the well-known Zachary's karate club network \citep{zachary1977information} as shown in Figure \ref{fig:karate_1}, which contains 34 nodes of 2 classes and 78 unweighted edges (grey solid lines).
		We then increase the weights of intra-class edges by ten times to simulate GCN-LPA.
		We find that GCN works well on this network (Figure \ref{fig:karate_2}), but GCN-LPA performs even better than GCN because the node embeddings are completely linearly separable as shown in Figure \ref{fig:karate_3}.
		To further justify our claim, we randomly add 20 ``noisy'' inter-class edges (grey dotted lines) to the original network, from which we observe that GCN is misled by noise and mixes nodes of two classes together (Figure \ref{fig:karate_4}), but GCN-LPA still distinguishes the two clusters (Figure \ref{fig:karate_5}) because it is better at ``denoising'' undesirable edges based on the supervised signal of labels.

\section{Connection to Existing Work}
\label{sec:existing_work}
	Edge weights play a key role in graph-based node classification as well as representation learning.
	In this section, we discuss three lines of related work that learn edge weights adaptively.
	
	\subsection{Locally Linear Embedding}
		Locally linear embedding (LLE) \citep{roweis2000nonlinear} and its variants \citep{zhang2007mlle, kong2012iterative} learn edge weights by constructing a linear dependency between a node and its neighbors, then use the learned edge weights to embed high-dimensional nodes into a low-dimensional space.
		Our work is similar to LLE in the aspect of transferring the knowledge of edge importance from one space to another, but the difference is that LLE is an unsupervised dimension reduction method that learns the graph structure based on local proximity only, while our work is semi-supervised and explores high-order relationship among nodes.
	
	\subsection{Label Propagation Algorithm}
		Classical LPA \citep{zhu2005semi,zhou2004learning} can only make use of node labels rather than node features.
		In contrast, adaptive LPA considers node features by making edge weights learnable.
		Typical techniques of learning edge weights include adopting kernel functions \citep{zhu2003semi,liu2018learning} (e.g., $a_{ij} = \exp ( -\sum_d (x_{id} - x_{jd})^2 / \sigma^2_d )$ where $d$ is dimensionality of features), minimizing neighborhood reconstruction error \citep{wang2008label,karasuyama2013manifold}, using leave-one-out loss \citep{zhang2007hyperparameter}, or imposing sparseness on edge weights \citep{hong2009sparsity}.
		However, in these LPA variants, node features are only used to assist learning the graph structure rather than explicitly mapped to node labels, which limits their capability in node classification.
		Another notable difference is that adaptive LPA learns edge weights by introducing the regularizations above, while our work takes LPA itself as regularization to learn edge weights.
	
	\subsection{Attention Mechanism on Graphs}
	Our method is also conceptually connected to attention mechanism on graphs \citep{velivckovic2018graph,thekumparampil2018attention,zhang2018gaan,liu2019geniepath}, in which an attention weight $\alpha_{ij}$ is learned between node $v_i$ and $v_j$.
	For example, $\alpha_{ij} = \text{LeakyReLU}({\bm a}^\top [W{\bf x}_i || W{\bf x}_j])$ in GAT \citep{velivckovic2018graph}, $\alpha_{ij} = a \cdot \cos(W {\bf x}_i, W {\bf x}_j)$ in AGNN \citep{thekumparampil2018attention}, $\alpha_{ij} = (W_1 {\bf x}_i)^\top W_2 {\bf x}_j$ in GaAN \citep{zhang2018gaan}, and $\alpha_{ij} = {\bm a}^\top \tanh(W_1 {\bf x}_i + W_2 {\bf x}_j)$ in GeniePath \citep{liu2019geniepath}, where $a$ and $W$ are trainable variables.
	A significant difference between these attention mechanisms and our work is that attention weights are learned based merely on feature similarity, while we propose that edge weights should be consistent with the distribution of labels on the graph, which requires less handcrafting of the attention function and is more task-oriented.
	Nevertheless, all the above formulas for calculating attentions can also be used in our model as the implementation of edge weights.
	
		\begin{table*}[t]
			\centering
			\setlength{\tabcolsep}{12pt}
			\begin{tabular}{c|ccccc}
				\hline
				Method & \textbf{Cora} & \textbf{Citeseer} & \textbf{Pubmed} & \textbf{Coauthor-CS} & \textbf{Coauthor-Phy} \\
				\hline
				MLP & 64.6 $\pm$ 1.7 & 62.0 $\pm$ 1.8 & 85.9 $\pm$ 0.3 & 91.7 $\pm$ 1.4 & 94.1 $\pm$ 1.2 \\
				LR & 77.3 $\pm$ 1.8 & 71.2 $\pm$ 1.8 & 86.0 $\pm$ 0.6 & 91.1 $\pm$ 0.6 & 93.8 $\pm$ 1.1 \\
				\hline
				LPA & 85.3 $\pm$ 0.9 & 70.0 $\pm$ 1.7 & 82.6 $\pm$ 0.6 & 91.3 $\pm$ 0.2 & 94.9 $\pm$ 0.4  \\
				\hline
				GCN & 88.2 $\pm$ 0.8 & 77.3 $\pm$ 1.5 & 87.2 $\pm$ 0.4 & 93.6 $\pm$ 1.5 & 96.2 $\pm$ 0.2 \\
				GAT & 87.7 $\pm$ 0.3 & 76.2 $\pm$ 0.9 & 86.9 $\pm$ 0.5 & 93.8 $\pm$ 0.4 & 96.3 $\pm$ 0.7 \\
				JK-Net & \textbf{89.1} $\pm$ 1.2 & 78.3 $\pm$ 0.9 & 85.8 $\pm$ 1.1 & 92.4 $\pm$ 0.4 & 94.8 $\pm$ 0.4 \\
				GraphSAGE & 86.8 $\pm$ 1.9 & 75.2 $\pm$ 1.1 & 84.7 $\pm$ 1.6 & 92.6 $\pm$ 1.6 & 94.5 $\pm$ 1.1 \\
				\hline
				GCN-LPA & 88.5 $\pm$ 1.5 & \textbf{78.7} $\pm$ 0.6 & \textbf{87.8} $\pm$ 0.6 & \textbf{94.8} $\pm$ 0.4 & \textbf{96.9} $\pm$ 0.2 \\
				\hline
			\end{tabular}
			\caption{Mean and the $95\%$ confidence intervals of test set accuracy for all methods and datasets.}
			\label{table:random_split}
		\end{table*}

\section{Experiments}
	We evaluate our model and present its performance on five datasets including citation networks and coauthor networks.
	We also study the hyper-parameter sensitivity and provide training time analysis.

	\subsection{Datasets}
		We use the following five datasets in our experiments:
		
		\textbf{Citation networks}:
		We consider three citation network datasets \citep{sen2008collective}:
		Cora, Citeseer, and Pubmed.
		In these datasets, nodes correspond to documents, edges correspond to citation links, and each node has a sparse bag-of-words feature vector as well as a class label.
		
		\textbf{Coauthor networks}:
		We also use two co-authorship networks \citep{shchur2018pitfalls}, Coauthor-CS and Coauthor-Phy, based on Microsoft Academic Graph from the KDD Cup 2016 challenge.
		Here nodes are authors and an edge indicates that two authors co-authored a paper.
		Node features represent paper keywords for each author's papers, and class labels indicate most active fields of study for each author.
				
		Statistics of the five datasets are shown in Table \ref{table:statistics}.
		We also calculate the intra-class edge rate (the fraction of edges that connect two nodes within the same class), which is significantly higher than inter-class edge rate in all networks.
		The finding supports our claim in Section \ref{sec:discussion} that node classification benefits from intra-class edges in a graph.

	\subsection{Baselines}
		We compare against the following baselines in our experiments.
		The first two baselines only utilize node features, the third baseline only utilizes graph structure, while the rest of baselines are GNN-based methods utilizing both node features and graph structure as input.
		Hyper-parameters of baselines are set as default in Python packages or their open-source codes unless otherwise stated.
		\begin{itemize}
			\item
				\textbf{Multi-layer Perceptron (MLP)} and \textbf{Logistic Regression (LR)} are feature-based methods that do not consider the graph structure.
				We set solver=`lbfgs' for LR and hidden\_layer\_sizes=50 for MLP using Python sklearn package.
			\item
				\textbf{Label Propagation (LPA)} \citep{zhu2005semi}, on the other hand, only consider the graph structure and ignore node features.
				We set the iteration of LPA as 20 in our implementation.
			\item
				\textbf{Graph Convolutional Network (GCN)} \citep{kipf2017semi} proposes a first-order approximation to spectral graph convolutions.
			\item
				\textbf{Graph Attention Network (GAT)} \citep{velivckovic2018graph} propose an attention mechanism to treat neighbors differently in the aggregation step.
			\item
				\textbf{Jumping Knowledge Networks (JK-Net)} \citep{xu2018representation} leverages different neighborhood ranges for each node to enable structure-aware representation. We use concat as the aggregator for JK-Net.
			\item
				\textbf{Graph Sampling and Aggregation (GraphSAGE)} \citep{hamilton2017inductive} is a mini-batch implementation of GCN that uses neighborhood sampling strategy and different aggregation schemes. We use mean as the aggregator for GraphSAGE.
		\end{itemize}
		
		\begin{figure*}[t]
		    \begin{minipage}[t]{0.3\linewidth}
		        \centering 
    			\includegraphics[width=\textwidth]{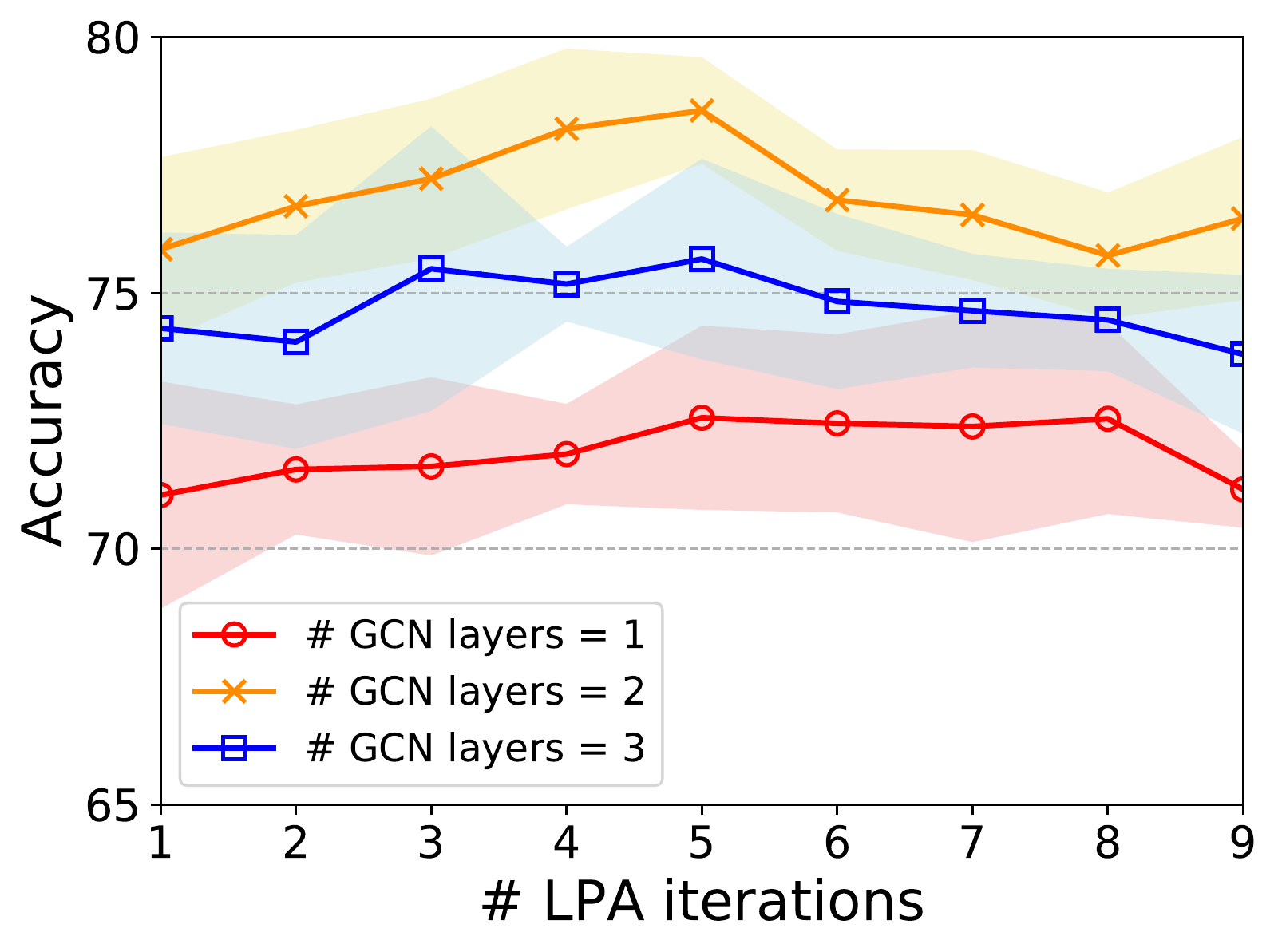}
    			\caption{Sensitivity to \# LPA iterations on Citeseer dataset.} 
    			\label{fig:ps_lpa_iter}
  			\end{minipage}
  			\hfill
		    \begin{minipage}[t]{0.3\linewidth} 
    			\centering 
    			\includegraphics[width=\textwidth]{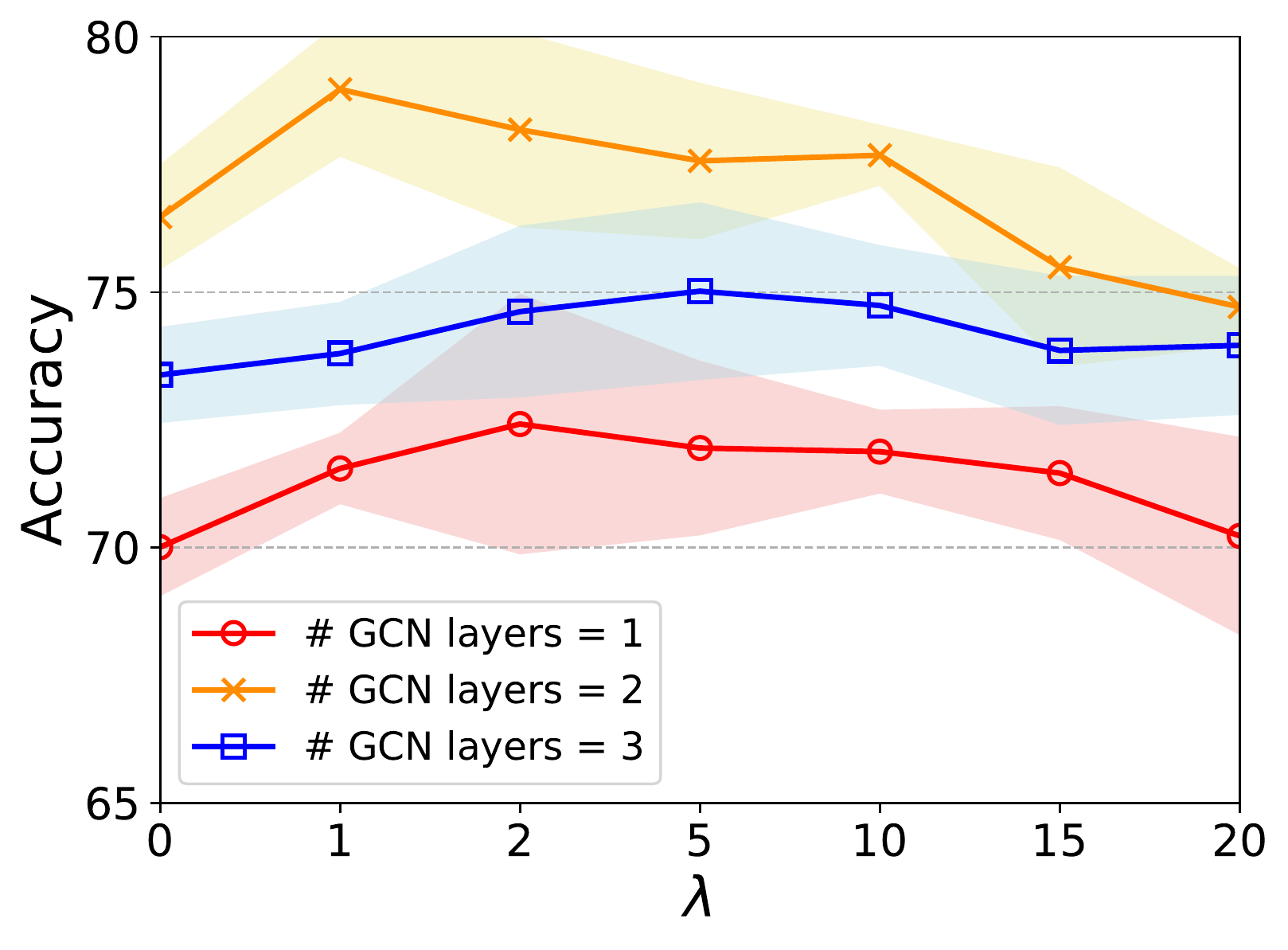}
    			\caption{Sensitivity to $\lambda$ on Citeseer dataset.} 
    			\label{fig:ps_lambda}
  			\end{minipage}
  			\hfill
  			\begin{minipage}[t]{0.3\linewidth} 
    			\centering 
    			\includegraphics[width=\textwidth]{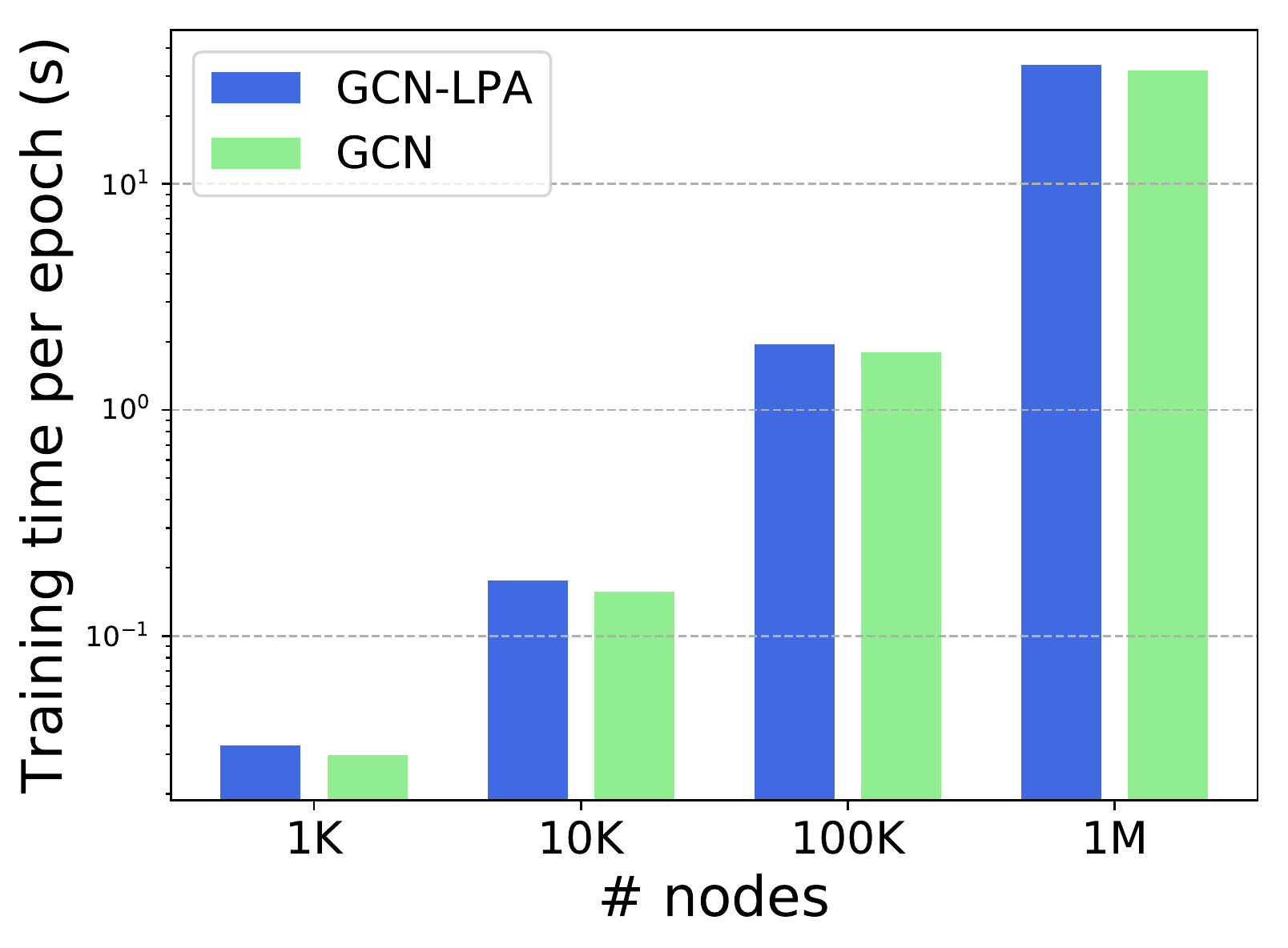}
    			\caption{Training time per epoch on random graphs.} 
    			\label{fig:training_time}
  			\end{minipage}
		\end{figure*}
		
		\begin{table*}[t]
			\centering
			\setlength{\tabcolsep}{8pt}
			\begin{tabular}{c|cccccc}
				\hline
				Ratio of labeled nodes & 0\% & 20\% & 40\% & 60\% & 80\% & 100\% \\
				\hline
				Accuracy & 75.8 $\pm$ 1.0 & 76.3 $\pm$ 1.1 & 76.7 $\pm$ 0.8 & 77.3 $\pm$ 0.7 & 78.1 $\pm$ 0.6 & 78.7 $\pm$ 0.6 \\
				\hline
			\end{tabular}
			\caption{Result of GCN-LPA on Citeseer dataset with differet ratio of labeled nodes in LPA.}
			\label{table:ratio}
		\end{table*}

	\subsection{Experimental Setup}
		Our experiments focus on the transductive setting where we only know labels of part of nodes but have access to the entire graph as well as features of all nodes.\footnote{The experimental setting here is the same as GCN \citep{kipf2017semi}. But note that our method can be easily generalized to inductive case if implemented in a way similar to GraphSAGE \citep{hamilton2017inductive}.}
		The ratio of training, validation, and test set are set as $6:2:2$.
		The weight of each edge is treated as a free variable during training.
		We train our model for 200 epochs using Adam \citep{kingma2015adam} and report the test set accuracy when validation set accuracy is maximized.
		Each experiment is repeated three times and we report the mean and the $95\%$ confidence interval.
		We initialize weights according to \cite{glorot2010understanding} and row-normalize input features.
		During training, we apply L2 regularization to the transformation matrices and use the dropout technique \citep{nitish2014dropout}.
		The settings of all other hyper-parameters can be found in Appendix \ref{app:f}.

	\subsection{Results}
		The results of node classification are summarized in Table \ref{table:random_split}.
		Table \ref{table:random_split} indicates that only using node features (MLP, LR) or graph structure (LPA) will lead to information loss and cannot fully exploit datasets in general.
		The results demonstrate that our proposed GCN-LPA model surpasses state-of-the-art GCN/GNN baselines.
		We note that JK-Net is a strong baseline on Cora, but it does not perform consistently well on other datasets.
		
		We investigate the influence of the number of LPA iterations and the training weight of LPA loss term $\lambda$ on the performance of classification.
		The results on Citeseer dataset are plotted in Figures \ref{fig:ps_lpa_iter} and \ref{fig:ps_lambda}, respectively, where each line corresponds to a given number of GCN layers in GCN-LPA.
		From Figure \ref{fig:ps_lpa_iter} we observe that the performance is boosted at first when the number of LPA iterations increases, then the accuracy stops increasing and decreases since a large number of LPA iterations will include more noisy nodes.
		Figure \ref{fig:ps_lambda} shows that training without the LPA loss term (i.e., $\lambda=0$) is more difficult than the case where $\lambda = 1 \sim 5$, which justifies our aforementioned claim that it is hard for the GCN part to learn both transformation matrices $W$ and edge weights $A$ simultaneously without the assistance of LPA regularization.
		
		To further show how much the LPA impacts the performance, we vary the ratio of labeled nodes in LPA from $100\%$ to $0\%$ during training, and report the result of acuracy on Citeseer dataset in Table \ref{table:ratio}.
		From Table \ref{table:ratio} we observe that the performance of GCN-LPA gets worse when the ratio of labeled nodes in LPA decreases.
		In addition, using more labeled nodes in LPA also helps improve the model stability.
		Note that a ratio of $0\%$ does not mean that GCN-LPA is equivalent to GCN \citep{kipf2017semi} because the edge weights in GCN-LPA is still trainable, which increases the risk of overfitting the training data.
		
		We study the training time of GCN-LPA on random graphs.
		We use the one-hot identity vector as feature and 0 as label for each node.
		The size of training set and validation set is 100 and 200, respectively, while the rest is test set.
		The average number of neighbors for each node is set as 5, and the number of nodes is varied from one thousand to one million.
		We run GCN-LPA and GCN for 100 epochs on a Microsoft Azure virtual machine with 1 NVIDIA Tesla M60 GPU, 12 Intel Xeon CPUs (E5-2690 v3 @2.60GHz), and 128GB of RAM, using the same hyper-parameter setting as in Cora.
		The training time per epoch of GCN-LPA and GCN is presented in Figure \ref{fig:training_time}.
		Our result shows that GCN-LPA requires only $9.2\%$ extra training time on average compared to GCN.

\section{Conclusion and Future Work}
	We studied the theoretical relationship between two types of well-known graph-based models for node classification, Label Propagation Algorithm and Graph Convolutional Neural Networks, from the perspectives of feature/label smoothing and feature/label influence.
	We then propose a unified model GCN-LPA, which learns transformation matrices and edge weights simultaneously in GCN with the assistance of LPA regularizer.
	We also analyze why our unified model performs better than traditional GCN in node classification.
	Experiments on five datasets demonstrate that our model outperforms state-of-the-art baselines, and it is also highly time-efficient with respect to the size of a graph.
	
	We point out two avenues of possible directions for future work.
	First, our proposed model focuses on transductive setting where all node features and the entire graph structure are given.
	An interesting problem is how the model performs in inductive setting where we have no access to test nodes during training.
	Second, the question of how to generalize the idea of our model to GNNs with different aggregation functions (e.g., concatenation or max-pooling) is also a promising direction.

\bibliography{reference}

\begin{thebibliography}{30}
\providecommand{\natexlab}[1]{#1}
\providecommand{\url}[1]{\texttt{#1}}
\expandafter\ifx\csname urlstyle\endcsname\relax
  \providecommand{\doi}[1]{doi: #1}\else
  \providecommand{\doi}{doi: \begingroup \urlstyle{rm}\Url}\fi

\bibitem[Glorot \& Bengio(2010)Glorot and Bengio]{glorot2010understanding}
Glorot, X. and Bengio, Y.
\newblock Understanding the difficulty of training deep feedforward neural
  networks.
\newblock In \emph{Proceedings of the 13th International Conference on
  Artificial Intelligence and Statistics}, 2010.

\bibitem[Gong et~al.(2017)Gong, Tao, Liu, Liu, and Yang]{gong2017label}
Gong, C., Tao, D., Liu, W., Liu, L., and Yang, J.
\newblock Label propagation via teaching-to-learn and learning-to-teach.
\newblock \emph{IEEE Transactions on Neural Networks and Learning Lystems},
  28\penalty0 (6), 2017.

\bibitem[Hamilton et~al.(2017)Hamilton, Ying, and
  Leskovec]{hamilton2017inductive}
Hamilton, W., Ying, Z., and Leskovec, J.
\newblock Inductive representation learning on large graphs.
\newblock In \emph{Advances in Neural Information Processing Systems}, 2017.

\bibitem[Hong et~al.(2009)Hong, Liu, and Yang]{hong2009sparsity}
Hong, C., Liu, Z., and Yang, J.
\newblock Sparsity induced similarity measure for label propagation.
\newblock In \emph{Proceedings of the 12th IEEE International Conference on
  Computer Vision}. IEEE, 2009.

\bibitem[Karasuyama \& Mamitsuka(2013)Karasuyama and
  Mamitsuka]{karasuyama2013manifold}
Karasuyama, M. and Mamitsuka, H.
\newblock Manifold-based similarity adaptation for label propagation.
\newblock In \emph{Advances in Neural Information Processing Systems}, 2013.

\bibitem[Kingma \& Ba(2015)Kingma and Ba]{kingma2015adam}
Kingma, D.~P. and Ba, J.
\newblock Adam: A method for stochastic optimization.
\newblock In \emph{Proceedings of the 3rd International Conference on Learning
  Representations}, 2015.

\bibitem[Kipf \& Welling(2017)Kipf and Welling]{kipf2017semi}
Kipf, T.~N. and Welling, M.
\newblock Semi-supervised classification with graph convolutional networks.
\newblock In \emph{Proceedings of the 5th International Conference on Learning
  Representations}, 2017.

\bibitem[Koh \& Liang(2017)Koh and Liang]{koh2017understanding}
Koh, P.~W. and Liang, P.
\newblock Understanding black-box predictions via influence functions.
\newblock In \emph{Proceedings of the 34th International Conference on Machine
  Learning}, 2017.

\bibitem[Kong et~al.(2012)Kong, Ding, Huang, and Nie]{kong2012iterative}
Kong, D., Ding, C., Huang, H., and Nie, F.
\newblock An iterative locally linear embedding algorithm.
\newblock In \emph{Proceedings of the 29th International Coference on
  International Conference on Machine Learning}. Omnipress, 2012.

\bibitem[Li et~al.(2018)Li, Han, and Wu]{li2018deeper}
Li, Q., Han, Z., and Wu, X.-M.
\newblock Deeper insights into graph convolutional networks for semi-supervised
  learning.
\newblock In \emph{The 32nd AAAI Conference on Artificial Intelligence}, 2018.

\bibitem[Liao et~al.(2019)Liao, Zhao, Urtasun, and Zemel]{liao2019lanczosnet}
Liao, R., Zhao, Z., Urtasun, R., and Zemel, R.~S.
\newblock Lanczosnet: Multi-scale deep graph convolutional networks.
\newblock In \emph{Proceedings of the 7th International Conference on Learning
  Representations}, 2019.

\bibitem[Liu et~al.(2019{\natexlab{a}})Liu, Lee, Park, Kim, Yang, Hwang, and
  Yang]{liu2018learning}
Liu, Y., Lee, J., Park, M., Kim, S., Yang, E., Hwang, S.~J., and Yang, Y.
\newblock Learning to propagate labels: Transductive propagation network for
  few-shot learning.
\newblock In \emph{Proceedings of the 7th International Conference on Learning
  Representations}, 2019{\natexlab{a}}.

\bibitem[Liu et~al.(2019{\natexlab{b}})Liu, Chen, Li, Zhou, Li, Song, and
  Qi]{liu2019geniepath}
Liu, Z., Chen, C., Li, L., Zhou, J., Li, X., Song, L., and Qi, Y.
\newblock Geniepath: Graph neural networks with adaptive receptive paths.
\newblock In \emph{The 33rd AAAI Conference on Artificial Intelligence},
  2019{\natexlab{b}}.

\bibitem[Qu et~al.(2019)Qu, Bengio, and Tang]{qu2019gmnn}
Qu, M., Bengio, Y., and Tang, J.
\newblock Gmnn: Graph markov neural networks.
\newblock In \emph{Proceedings of the 36th International Conference on Machine
  Learning}, 2019.

\bibitem[Roweis \& Saul(2000)Roweis and Saul]{roweis2000nonlinear}
Roweis, S.~T. and Saul, L.~K.
\newblock Nonlinear dimensionality reduction by locally linear embedding.
\newblock \emph{science}, 290\penalty0 (5500), 2000.

\bibitem[Sen et~al.(2008)Sen, Namata, Bilgic, Getoor, Galligher, and
  Eliassi-Rad]{sen2008collective}
Sen, P., Namata, G., Bilgic, M., Getoor, L., Galligher, B., and Eliassi-Rad, T.
\newblock Collective classification in network data.
\newblock \emph{AI magazine}, 29\penalty0 (3), 2008.

\bibitem[Shchur et~al.(2018)Shchur, Mumme, Bojchevski, and
  G{\"u}nnemann]{shchur2018pitfalls}
Shchur, O., Mumme, M., Bojchevski, A., and G{\"u}nnemann, S.
\newblock Pitfalls of graph neural network evaluation.
\newblock In \emph{Neural Information Processing Systems Workshop on Relational
  Representation Learning}, 2018.

\bibitem[Srivastava et~al.(2014)Srivastava, Hinton, Krizhevsky, Sutskever, and
  Salakhutdinov]{nitish2014dropout}
Srivastava, N., Hinton, G., Krizhevsky, A., Sutskever, I., and Salakhutdinov,
  R.
\newblock Dropout: a simple way to prevent neural networks from overfitting.
\newblock \emph{The Journal of Machine Learning Research}, 15\penalty0 (1),
  2014.

\bibitem[Thekumparampil et~al.(2018)Thekumparampil, Wang, Oh, and
  Li]{thekumparampil2018attention}
Thekumparampil, K.~K., Wang, C., Oh, S., and Li, L.-J.
\newblock Attention-based graph neural network for semi-supervised learning.
\newblock \emph{arXiv preprint arXiv:1803.03735}, 2018.

\bibitem[Veli{\v{c}}kovi{\'c} et~al.(2018)Veli{\v{c}}kovi{\'c}, Cucurull,
  Casanova, Romero, Lio, and Bengio]{velivckovic2018graph}
Veli{\v{c}}kovi{\'c}, P., Cucurull, G., Casanova, A., Romero, A., Lio, P., and
  Bengio, Y.
\newblock Graph attention networks.
\newblock In \emph{Proceedings of the 6th International Conference on Learning
  Representations}, 2018.

\bibitem[Wang \& Zhang(2008)Wang and Zhang]{wang2008label}
Wang, F. and Zhang, C.
\newblock Label propagation through linear neighborhoods.
\newblock \emph{IEEE Transactions on Knowledge and Data Engineering},
  20\penalty0 (1), 2008.

\bibitem[Xu et~al.(2018)Xu, Li, Tian, Sonobe, Kawarabayashi, and
  Jegelka]{xu2018representation}
Xu, K., Li, C., Tian, Y., Sonobe, T., Kawarabayashi, K.-i., and Jegelka, S.
\newblock Representation learning on graphs with jumping knowledge networks.
\newblock In \emph{Proceedings of the 35th International Conference on Machine
  Learning}, 2018.

\bibitem[Xu et~al.(2019)Xu, Hu, Leskovec, and Jegelka]{xu2019powerful}
Xu, K., Hu, W., Leskovec, J., and Jegelka, S.
\newblock How powerful are graph neural networks?
\newblock In \emph{Proceedings of the 7th International Conference on Learning
  Representations}, 2019.

\bibitem[Zachary(1977)]{zachary1977information}
Zachary, W.~W.
\newblock An information flow model for conflict and fission in small groups.
\newblock \emph{Journal of anthropological research}, 33\penalty0 (4), 1977.

\bibitem[Zhang et~al.(2018)Zhang, Shi, Xie, Ma, King, and Yeung]{zhang2018gaan}
Zhang, J., Shi, X., Xie, J., Ma, H., King, I., and Yeung, D.-Y.
\newblock Gaan: Gated attention networks for learning on large and
  spatiotemporal graphs.
\newblock \emph{arXiv preprint arXiv:1803.07294}, 2018.

\bibitem[Zhang \& Lee(2007)Zhang and Lee]{zhang2007hyperparameter}
Zhang, X. and Lee, W.~S.
\newblock Hyperparameter learning for graph based semi-supervised learning
  algorithms.
\newblock In \emph{Advances in Neural Information Processing Systems}, 2007.

\bibitem[Zhang \& Wang(2007)Zhang and Wang]{zhang2007mlle}
Zhang, Z. and Wang, J.
\newblock Mlle: Modified locally linear embedding using multiple weights.
\newblock In \emph{Advances in Neural Information Processing Systems}, 2007.

\bibitem[Zhou et~al.(2004)Zhou, Bousquet, Lal, Weston, and
  Sch{\"o}lkopf]{zhou2004learning}
Zhou, D., Bousquet, O., Lal, T.~N., Weston, J., and Sch{\"o}lkopf, B.
\newblock Learning with local and global consistency.
\newblock In \emph{Advances in Neural Information Processing Systems}, 2004.

\bibitem[Zhu et~al.(2003)Zhu, Ghahramani, and Lafferty]{zhu2003semi}
Zhu, X., Ghahramani, Z., and Lafferty, J.~D.
\newblock Semi-supervised learning using gaussian fields and harmonic
  functions.
\newblock In \emph{Proceedings of the 20th International Conference on Machine
  Learning}, 2003.

\bibitem[Zhu et~al.(2005)Zhu, Lafferty, and Rosenfeld]{zhu2005semi}
Zhu, X., Lafferty, J., and Rosenfeld, R.
\newblock \emph{Semi-supervised learning with graphs}.
\newblock PhD thesis, Carnegie Mellon University, school of language
  technologies institute, 2005.

\end{thebibliography}
\bibliographystyle{icml2020}

\clearpage
\renewcommand\thesubsection{\Alph{subsection}}
\onecolumn

\section*{Appendix}
	\subsection{Proof of Theorem \ref{thm:smoothing}}
	\label{app:a}
		\begin{proof}
			Denote $\tilde a_{ij} = a_{ij} / d_{ii}$ as the normalized weight of edge $(j, i)$.
			It is clear that $\sum_{j \in \mathcal N(i)} \tilde a_{ij} = 1$.
			Given that $\mathcal M$ is differentiable, we perform a first-order Taylor expansion with Peano's form of remainder at ${\bf x}_i$ for $\sum_{j \in \mathcal N(i)} \tilde a_{ij} y_j$:
			\begin{equation}
			\begin{split}
				& \sum_{j \in \mathcal N(i)} \tilde a_{ij} y_j = \sum_{j \in \mathcal N(i)} \tilde a_{ij} \mathcal M({\bf x}_j) \\
				= \ & \sum_{j \in \mathcal N(i)} \tilde a_{ij} \left( \mathcal M({\bf x}_i) + \frac{\partial \mathcal M({\bf x}_i)}{\partial {\bf x}^\top} ({\bf x}_j - {\bf x}_i) + o(\| {\bf x}_j - {\bf x}_i \|_2) \right) \\
				= \ & \mathcal M({\bf x}_i) + \frac{\partial \mathcal M({\bf x}_i)}{\partial {\bf x}^\top} \sum_{j \in \mathcal N(i)} \tilde a_{ij} ({\bf x}_j - {\bf x}_i) + \sum_{j \in \mathcal N(i)} \tilde a_{ij} o(\| {\bf x}_j - {\bf x}_i \|_2) \\
				= \ & y_i - \frac{\partial \mathcal M({\bf x}_i)}{\partial {\bf x}^\top} \epsilon_i + \sum_{j \in \mathcal N(i)} \tilde a_{ij} o(\| {\bf x}_j - {\bf x}_i \|_2).
			\end{split}
			\end{equation}
			According to Cauchy-Schwarz inequality and $L$-Lipschitz property, we have
			\begin{equation}
				\bigg| \frac{\partial \mathcal M({\bf x}_i)}{\partial {\bf x}^\top} \epsilon_i \bigg| \ \leq \ \bigg\| \frac{\partial \mathcal M({\bf x}_i)}{\partial {\bf x}^\top} \bigg\|_2 \| \epsilon_i \|_2 \ \leq \ L \| \epsilon_i \|_2.
			\end{equation}
			Therefore, the approximation of $y_i$ is bounded by
			\begin{equation}
			\begin{split}
				& \bigg| y_i - \sum_{j \in \mathcal N(i)} \tilde a_{ij} y_j \bigg| \\
				= \ & \bigg| \frac{\partial \mathcal M({\bf x}_i)}{\partial {\bf x}^\top} \epsilon_i - \sum_{j \in \mathcal N(i)} \tilde a_{ij} o(\| {\bf x}_j - {\bf x}_i \|_2) \bigg | \\
				\leq \ & \bigg| \frac{\partial \mathcal M({\bf x}_i)}{\partial {\bf x}^\top} \epsilon_i \bigg| + \bigg| \sum_{j \in \mathcal N(i)} \tilde a_{ij} o(\| {\bf x}_j - {\bf x}_i \|_2) \bigg | \\
				\leq \ & L \| \epsilon_i\|_2 + o \big( \max_{j \in \mathcal N(i)} ( \| {\bf x}_j - {\bf x}_i \|_2 ) \big).
			\end{split}
			\end{equation}
		\end{proof}

	\subsection{Proof of Theorem \ref{thm:influence}}
	\label{app:b}
		Before proving Theorem \ref{thm:influence}, we first give two lemmas that demonstrate the exact form of feature influence and label influence defined in this paper.
		The relationship between feature influence and label influence can then be deduced from their exact forms.
		\begin{lemma}
		\label{lemma:1}
			Assume that the nonlinear activation function in GCN is ReLU.
			Let $\mathcal P_k^{a \rightarrow b}$ be a path $[v^{(k)}, v^{(k-1)},\cdots, v^{(0)}]$ of length $k$ from node $v_a$ to node $v_b$, where $v^{(k)} = v_a$, $v^{(0)} = v_b$, and $v^{(i-1)} \in \mathcal N(v^{(i)})$ for $i = k,\cdots, 1$.
			Then we have
			\begin{equation}
			\label{eq:appendix_0}
				 \tilde I_f(v_a, v_b; k) = \sum_{\mathcal P_k^{a \rightarrow b}} \prod_{i=k}^1 \tilde a_{v^{(i-1)}, v^{(i)}},
			\end{equation}
			where $\tilde a_{v^{(i-1)}, v^{(i)}}$ is the normalized weight of edge $(v^{(i)}, v^{(i-1)})$.
		\end{lemma}
		
		\begin{proof}
			See \cite{xu2018representation} for the detailed proof.
		\end{proof}
		
		The product term in Eq. (\ref{eq:appendix_0}) is the probability of a given path $\mathcal P_k^{a \rightarrow b}$.
		Therefore, the right hand side in Eq. (\ref{eq:appendix_0}) is the sum over probabilities of all possible paths of length $k$ from $v_a$ to $v_b$, which is the probability that a random walk starting at $v_a$ ends at $v_b$ after taking $k$ steps.
		
		\begin{lemma}
		\label{lemma:2}
			Let $\mathcal U_j^{a \rightarrow b}$ be a path $[v^{(j)}, v^{(j-1)},\cdots, v^{(0)}]$ of length $j$ from node $v_a$ to node $v_b$, where $v^{(j)} = v_a$, $v^{(0)} = v_b$, $v^{(i-1)} \in \mathcal N(v^{(i)})$ for $i = j,\cdots, 1$, and all nodes along the path are unlabeled except $v^{(0)}$.
			Then we have
			\begin{equation}
				I_l(v_a, v_b; k) = \sum_{j=1}^k \sum_{\mathcal U_j^{a \rightarrow b}} \prod_{i=j}^1 \tilde a_{v^{(i-1)}, v^{(i)}},
			\end{equation}
			where $\tilde a_{v^{(i-1)}, v^{(i)}}$ is the normalized weight of edge $(v^{(i)}, v^{(i-1)})$.
		\end{lemma}
		
		To intuitively understand this lemma, note that there are two differences between Lemma \ref{lemma:1} and Lemma \ref{lemma:2}:
		(1) In Lemma \ref{lemma:1}, $\tilde I_f(v_a, v_b; k)$ sums over all paths from $v_a$ to $v_b$ of length $k$, but in Lemma \ref{lemma:2}, $I_l(v_a, v_b; k)$ sums over all paths from $v_a$ to $v_b$ of length no more than $k$.
		The is because in LPA, $v_b$'s label is reset to its initial value after each iteration, which means that the label of $v_b$ serves as a constant signal that begins propagating in the graph again and again after each iteration.
		(2) In Lemma \ref{lemma:1} we consider all possible paths from $v_a$ to $v_b$, but in Lemma \ref{lemma:2}, the paths are restricted to contain unlabeled nodes only.
		The reason here is the same as above:
		Since the labels of labeled nodes are reset to their initial values after each iteration in LPA, the influence of $v_b$'s label will be absorbed in labeled nodes, and the propagation of $v_b$'s label will be cut off at these nodes.
		Therefore, $v_b$'s label can only flow to $v_a$ along the paths with unlabeled nodes only.
		See Figure \ref{fig:lemma_2} for an illustrating example showing the label propagation in LPA.
		
		\begin{figure}[t]
		\centering
        \begin{subfigure}[b]{0.2\textwidth}
            \includegraphics[width=\textwidth]{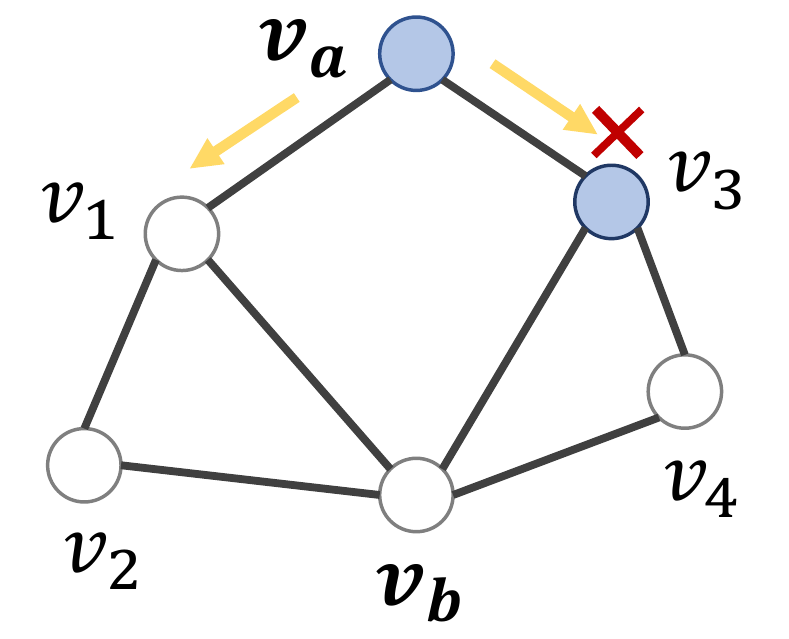}
            \caption{Iteration 1}
            \label{fig:d1}
        \end{subfigure}
        \hfill
        \begin{subfigure}[b]{0.2\textwidth}
            \includegraphics[width=\textwidth]{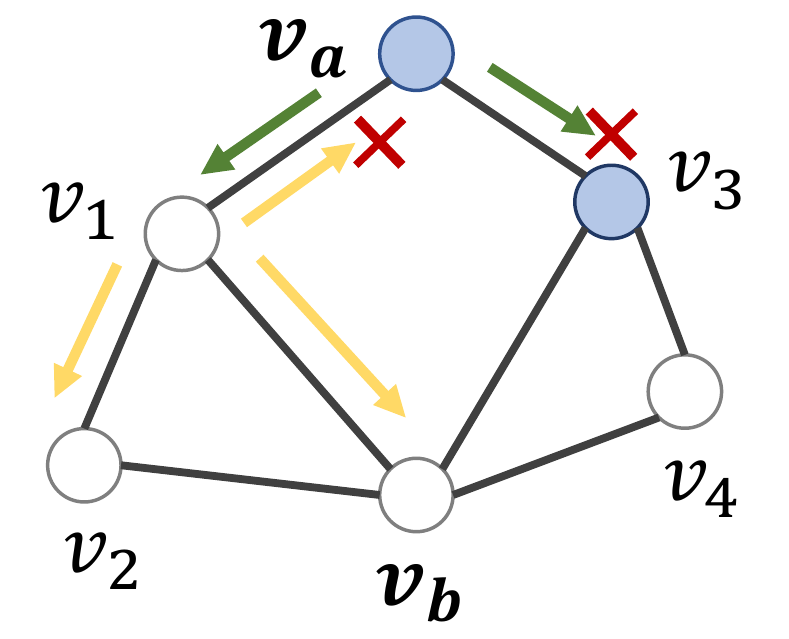}
            \caption{Iteration 2}
            \label{fig:d2}
        \end{subfigure}
        \hfill
        \begin{subfigure}[b]{0.2\textwidth}
            \includegraphics[width=\textwidth]{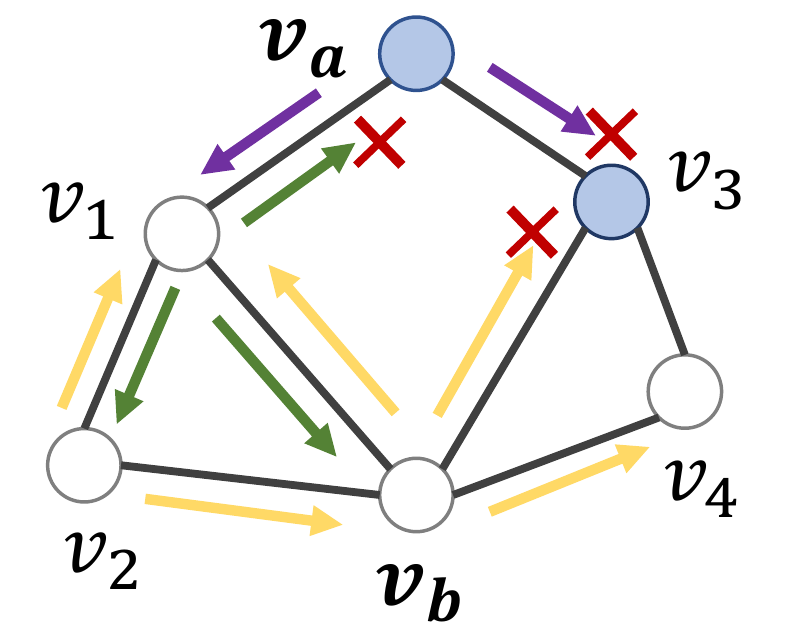}
            \caption{Iteration 3}
            \label{fig:d2}
        \end{subfigure}
        \hfill
        \begin{subfigure}[b]{0.2\textwidth}
            \includegraphics[width=\textwidth]{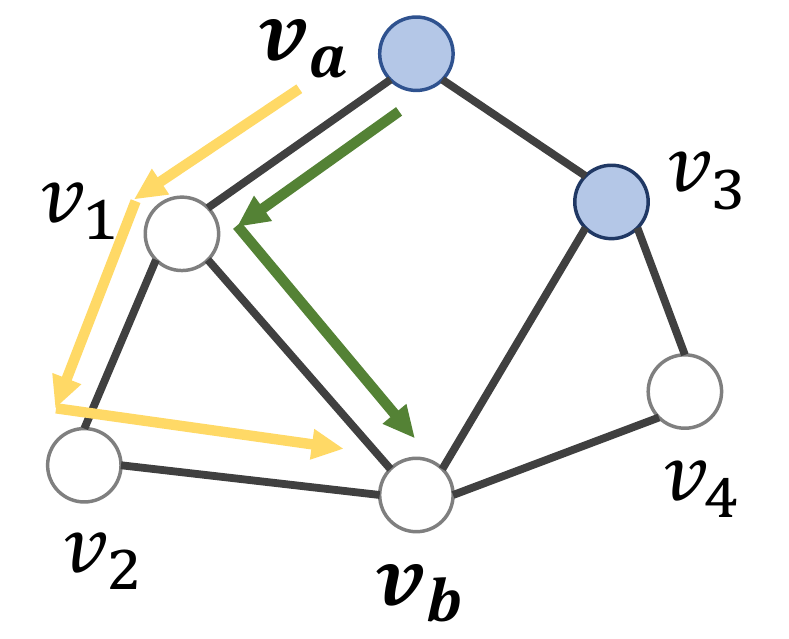}
            \caption{Paths from $v_a$ to $v_b$}
            \label{fig:d2}
        \end{subfigure}
        \caption{An illustrating example of label propagation in LPA. Suppose labels are propagated for three iterations, and no self-loop exists. Blue nodes are labeled while white nodes are unlabeled. (a) $v_a$'s label propagates to $v_1$ (yellow arrows). Note that the propagation of $v_a$'s label to $v_3$ is cut off since $v_3$ is labeled thus absorbing $v_a$'s label. (b) $v_a$'s label that propagated to $v_1$ further propagates to $v_2$ and $v_b$ (yellow arrows). Meanwhile, $v_a$'s label is reset to its initial value then propagates from $v_a$ again (green arrows). (c) Label propagation in iteration 3. Purple arrows denote the propagation of $v_a$'s label starting from $v_a$ for the third time. (d) All possible paths of length no more than three from $v_a$ to $v_b$ containing unlabeled nodes only. Note that there is no path of length one from $v_a$ to $v_b$.}
        \label{fig:lemma_2}
    	\end{figure}
		
		\begin{proof}
			As mentioned above, a significant difference between LPA and GCN is that all labeled nodes are reset to its original labels after each iteration in LPA.
			This implies that the initial label $y_b$ of node $v_b$ appears not only as $y_b^{(0)}$, but also as every $y_b^{(j)}$ for $j = 1,\cdots, k-1$.
			Therefore, the influence of $y_b$ on $y_a^{(k)}$ is the cumulative influence of $y_b^{(j)}$ on $y_a^{(k)}$ for $j = 0, 1,\cdots, k-1$:
			\begin{equation}
			\label{eq:appendix_1}
				I_l(v_a, v_b; k) = \frac{\partial y_a^{(k)}}{\partial y_b} = \sum_{j=0}^{k-1} \frac{\partial y_a^{(k)}}{\partial y_b^{(j)}}.
			\end{equation}
			According to the updating rule of LPA, we have
			\begin{equation}
			\label{eq:decompose}
				\frac{\partial y_a^{(k)}}{\partial y_b^{(j)}} = \frac{\partial \sum_{v_z \in \mathcal N(v_a)} \tilde a_{az} y_z^{(k-1)}}{\partial y_b^{(j)}} = \sum_{v_z \in \mathcal N(v_a)} \tilde a_{az} \frac{\partial y_z^{(k-1)}}{\partial y_b^{(j)}}.
			\end{equation}
			In the above equation, the derivative $\frac{\partial y_a^{(k)}}{\partial y_b^{(j)}}$ is decomposed into the weighted average of $\frac{\partial y_z^{(k-1)}}{\partial y_b^{(j)}}$, where $v_z$ traverses all neighbors of $v_a$.
			For those $v_z$'s that are initially labeled, $y_z^{(k-1)}$ is reset to their initial labels in each iteration.
			Therefore, they are always constant and independent of $y_b^{(j)}$, meaning that their derivatives w.r.t. $y_b^{(j)}$ are zero.
			So we only need to consider the terms where $v_z$ is an unlabeled node:
			\begin{equation}
			\label{eq:expansion}
				\frac{\partial y_a^{(k)}}{\partial y_b^{(j)}} = \sum_{v_z \in \mathcal N(v_a), z > m} \tilde a_{az} \frac{\partial y_z^{(k-1)}}{\partial y_b^{(j)}},
			\end{equation}
			where $z>m$ means $v_z$ is unlabeled.
			To intuitively understand Eq. (\ref{eq:expansion}), one can imagine that we perform a random walk starting from node $v_a$ for one step, where the ``transition probability'' is the edge weights $\tilde a$, and all nodes in this random walk are restricted to unlabeled nodes only.
			Note that we can further decompose every $y_z^{(k-1)}$ in Eq. (\ref{eq:expansion}) in the way similar to what we do for $y_a^{(k)}$ in Eq. (\ref{eq:decompose}).
			So the expansion in Eq. (\ref{eq:expansion}) can be performed iteratively until the index $k$ decreases to $j$.
			This is equivalent to performing all possible random walks for $k - j$ steps starting from $v_a$, where all nodes but the last in the random walk are restricted to be unlabeled nodes:
			\begin{equation}
			\label{eq:full}
				\frac{\partial y_a^{(k)}}{\partial y_b^{(j)}} = \sum_{v_z \in \mathcal V} \sum_{\mathcal U_{k-j}^{a \rightarrow z}} \left( \prod_{i=k-j}^1 \tilde a_{v^{(i-1)}, v^{(i)}} \right) \frac{\partial y_z^{(j)}}{\partial y_b^{(j)}},
			\end{equation}
			where $v_z$ in the first summation term is the end node of a random walk, $\mathcal U_{k-j}^{a \rightarrow z}$ in the second summation term is an unlabeled-nodes-only path from $v_a$ to $v_z$ of length $k-j$, and the product term is the probability of a given path $\mathcal U_{k-j}^{a \rightarrow z}$.
			Consider the last term $\frac{\partial y_z^{(j)}}{\partial y_b^{(j)}}$ in Eq. (\ref{eq:full}).
			We know that $\frac{\partial y_z^{(j)}}{\partial y_b^{(j)}} = 0$ for all $z \neq b$ and $\frac{\partial y_z^{(j)}}{\partial y_b^{(j)}} = 1$ for $z=b$, which means that only those random-walk paths that end exactly at $v_b$ (i.e., the end node $v_z$ is exactly $v_b$) count for the computation in Eq. (\ref{eq:full}).
			Therefore, we have
			\begin{equation}
			\label{eq:appendix_2}
				\frac{\partial y_a^{(k)}}{\partial y_b^{(j)}} = \sum_{\mathcal U_{k-j}^{a \rightarrow b}} \prod_{i=k-j}^1 \tilde a_{v^{(i-1)}, v^{(i)}},
			\end{equation}
			where $\mathcal U_{k-j}^{a \rightarrow b}$ is a path from $v_a$ to $v_b$ of length $k-j$ containing only unlabeled nodes except $v_b$.
			Substituting the right hand term of Eq. (\ref{eq:appendix_1}) with Eq. (\ref{eq:appendix_2}), we obtain that
			\begin{equation}
				I_l(v_a, v_b; k) = \sum_{j=0}^{k-1} \sum_{\mathcal U_{k-j}^{a \rightarrow b}} \prod_{i=k-j}^1 \tilde a_{v^{(i-1)}, v^{(i)}} = \sum_{j=1}^{k} \sum_{\mathcal U_j^{a \rightarrow b}} \prod_{i=j}^1 \tilde a_{v^{(i-1)}, v^{(i)}}.
			\end{equation}
		\end{proof}
		
		Now Theorem \ref{thm:influence} can be proved by combining Lemma \ref{lemma:1} and Lemma \ref{lemma:2}:
		
		\begin{proof}
			Suppose that whether a node is labeled or not is independent of each other for the given graph.
			Then we have
			\begin{equation}
			\begin{split}
				\mathbb E \big[ I_l(v_a, v_b; k) \big] = & \mathbb E \left[ \sum_{j=1}^{k} \sum_{\mathcal U_j^{a \rightarrow b}} \prod_{i=j}^1 \tilde a_{v^{(i-1)}, v^{(i)}} \right] = \sum_{j=1}^{k} \mathbb E \left[ \sum_{\mathcal U_j^{a \rightarrow b}} \prod_{i=j}^1 \tilde a_{v^{(i-1)}, v^{(i)}} \right]\\
				= & \sum_{j=1}^k \sum_{\mathcal P_j^{a \rightarrow b}} \Pr \big( \mathcal P_j^{a \rightarrow b} \text{ is an unlabeled-nodes-only path} \big) \prod_{i=j}^1 \tilde a_{v^{(i-1)}, v^{(i)}}\\
				= & \sum_{j=1}^k \sum_{\mathcal P_j^{a \rightarrow b}} \beta^j \prod_{i=j}^1 \tilde a_{v^{(i-1)}, v^{(i)}}\\
				= & \sum_{j=1}^k \beta^j \tilde I_f(v_a, v_b; j).
			\end{split}
			\end{equation}
		\end{proof}

	\subsection{Proof of Theorem \ref{thm:lpa}}
	\label{app:c}
		\begin{proof}
			Denote the set of labels as $\mathcal L$.
			Since different label dimensions in $y_a^{(\cdot)}$ do not interact with each other when running LPA, the value of the $y_a$-th dimension in $y_a^{(\cdot)}$ (denoted by $y_a^{(\cdot)}[y_a]$) comes only from the nodes with initial label $y_a$.
			It is clear that
			\begin{equation}
				y_a^{(k)}[y_a] = \sum_{v_b: y_b = y_a} \sum_{j=1}^k \sum_{\mathcal U_j^{a \rightarrow b}} \prod_{i=j}^1 \tilde a_{v^{(i-1)}, v^{(i)}},
			\end{equation}
			which equals $\sum_{v_b: y_b = y_a} I_l(v_a, v_b; k)$ according to Lemma \ref{lemma:2}.
			Therefore, we have
			\begin{equation}
				\Pr(\hat y_a = y_a) = \frac{y_a^{(k)}[y_a]}{\sum_{i \in \mathcal L} y_a^{(k)}[i]} \propto y_a^{(k)}[y_a] = \sum_{v_b: y_b = y_a} I_l(v_a, v_b; k)
			\end{equation}
		\end{proof}

	\subsection{Proof of Theorem \ref{thm:decrease}}
	\label{app:d}
		In this proof we assume that the dimension of node representations is one, but note that the conclusion can be easily generalized to the case of multi-dimensional representations since the function $D({\bf x})$ can be decomposed into the sum of one-dimensional cases.
		In the following of this proof, we still use bold notations ${\bf x}_i^{(k)}$ and ${\bf h}_i^{(k)}$ to denote node representations, but keep in mind that they are scalars rather than vectors.
		
		We give two lemmas before proving Theorem \ref{thm:decrease}.
		The first one is about the gradient of $D({\bf x})$:
		\begin{lemma}
		\label{lemma:3}
			${\bf h}_i^{(k)} = {\bf x}_i^{(k)} - \frac{\partial D({\bf x}^{(k)})}{\partial {\bf x}_i^{(k)}}$.
		\end{lemma}
		
		\begin{proof}
			${\bf x}_i^{(k)} - \frac{\partial D({\bf x}^{(k)})}{\partial {\bf x}_i^{(k)}} = {\bf x}_i^{(k)} - \sum_{v_j \in \mathcal N(v_i)} \tilde a_{ij} ({\bf x}_i^{(k)} - {\bf x}_j^{(k)}) = \sum_{v_j \in \mathcal N(v_i)} \tilde a_{ij} {\bf x}_j^{(k)} = {\bf h}_i^{(k)}$.
		\end{proof}
		
		It is interesting to see from Lemma \ref{lemma:3} that the aggregation step in GCN is equivalent to running gradient descent for one step with a step size of one.
		However, this is not able to guarantee that $D({\bf h}^{(k)}) \leq D({\bf x}^{(k)})$ because the step size may be too large to reduce the value of $D$.
		
		The second lemma is about the Hessian of $D({\bf x})$:
		
		\begin{lemma}
		\label{lemma:4}
			$\nabla^2 D({\bf x}) \preceq 2I$, or equivalently, $2I - \nabla^2 D({\bf x})$ is a positive semidefinite matrix.
		\end{lemma}
		
		\begin{proof}
			We first calculate the Hessian of $D({\bf x}) = \frac{1}{2} \sum_{v_i, v_j} \tilde a_{ij} \| {\bf x}_i - {\bf x}_j \|_2^2$:
			\begin{equation}
				\nabla^2 D({\bf x}) =
				\left[
					\begin{matrix}
						1-\tilde a_{11} & -\tilde a_{12} & \cdots & -\tilde a_{1n} \\
						-\tilde a_{21} & 1-\tilde a_{22} & \cdots & -\tilde a_{2n} \\
						\vdots & \vdots & \ddots & \vdots \\
						-\tilde a_{n1} & -\tilde a_{n2} & \cdots & 1-\tilde a_{nn} \\
					\end{matrix}
				\right]
				= I - D^{-1}A.
			\end{equation}
			Therefore, $2I - \nabla^2 D({\bf x}) = I + D^{-1}A$.
			Since $D^{-1}A$ is Markov matrix (i.e., each entry is non-negative and the sum of each row is one), its eigenvalues are within the range [-1, 1], so the eigenvalues of $I + D^{-1}A$ are within the range [0, 2].
			Therefore, $I + D^{-1}A$ is a positive semidefinite matrix, and we have $\nabla^2 D({\bf x}) \preceq 2I$.
		\end{proof}
		
		We can now prove Theorem \ref{thm:decrease}:
		
		\begin{proof}
			Since $D$ is a quadratic function, we perform a second-order Taylor expansion of $D$ around ${\bf x}^{{(k)}}$ and obtain the following inequality:
			\begin{equation}
			\begin{split}
				D({\bf h}^{(k)}) = & D({\bf x}^{(k)}) + \nabla D({\bf x}^{(k)})^\top ({\bf h}^{(k)} - {\bf x}^{(k)}) + \frac{1}{2} ({\bf h}^{(k)} - {\bf x}^{(k)})^\top \nabla^2 D({\bf x}) ({\bf h}^{(k)} - {\bf x}^{(k)}) \\
				= & D({\bf x}^{(k)}) - \nabla D({\bf x}^{(k)})^\top \nabla D({\bf x}^{(k)}) + \frac{1}{2} \nabla D({\bf x}^{(k)})^\top \nabla^2 D({\bf x}) \nabla D({\bf x}^{(k)}) \\
				\leq & D({\bf x}^{(k)}) - \nabla D({\bf x}^{(k)})^\top \nabla D({\bf x}^{(k)}) + \nabla D({\bf x}^{(k)})^\top \nabla D({\bf x}^{(k)}) = D({\bf x}^{(k)}).
			\end{split}
			\end{equation}
		\end{proof}
		
		\begin{figure}[h]
			\centering
			\begin{subfigure}[b]{\textwidth}
   				\includegraphics[width=\textwidth]{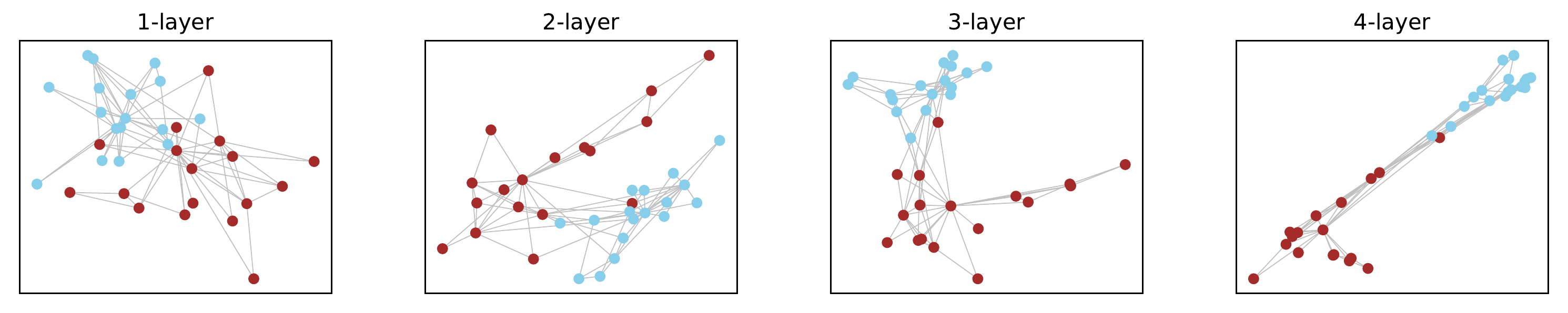}
   				\caption{GCN on the original network}
   				\label{fig:app_karate_1}
			\end{subfigure}
			\hfill
			\begin{subfigure}[b]{\textwidth}
				\includegraphics[width=\textwidth]{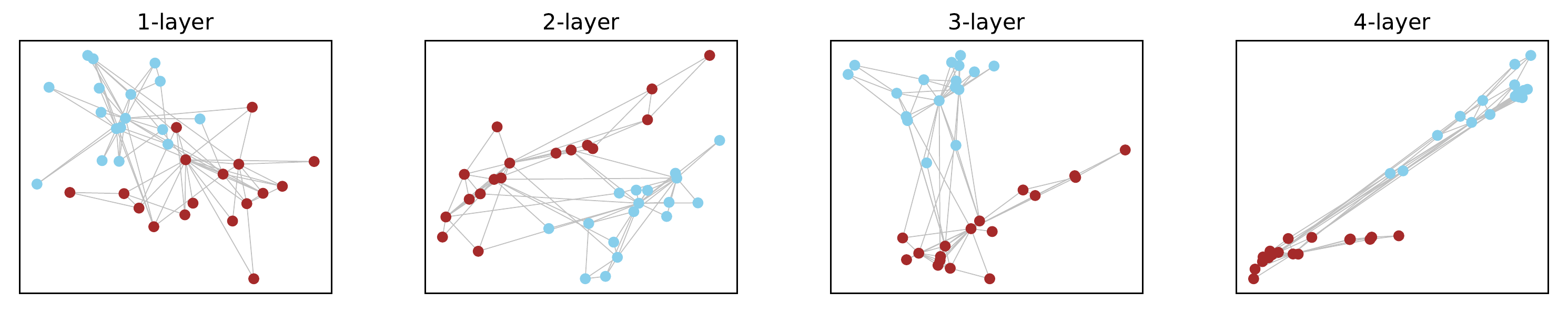}
				\caption{GCN-LPA on the original network}
				\label{fig:app_karate_2}
			\end{subfigure}
			\hfill
			\begin{subfigure}[b]{\textwidth}
   				\includegraphics[width=\textwidth]{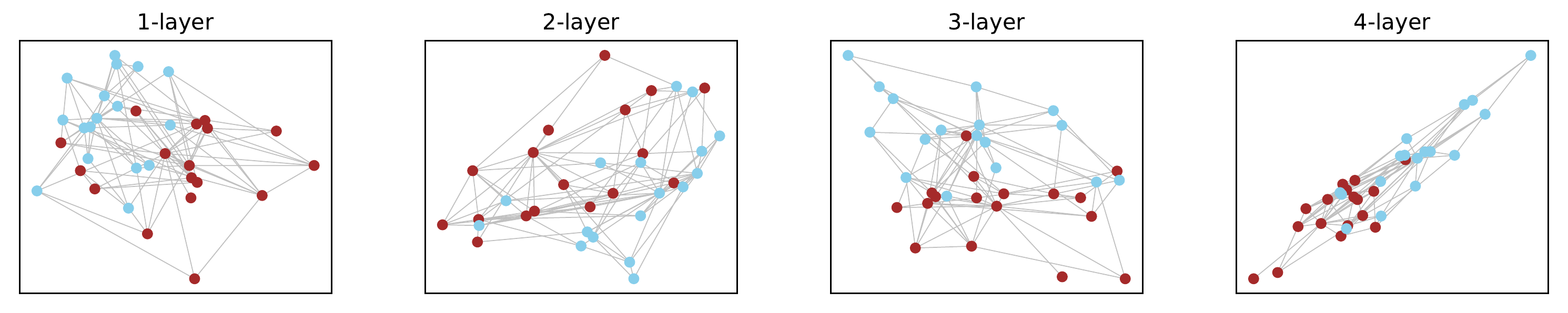}
   				\caption{GCN on the noisy network}
   				\label{fig:app_karate_3}
			\end{subfigure}
			\hfill
			\begin{subfigure}[b]{\textwidth}
				\includegraphics[width=\textwidth]{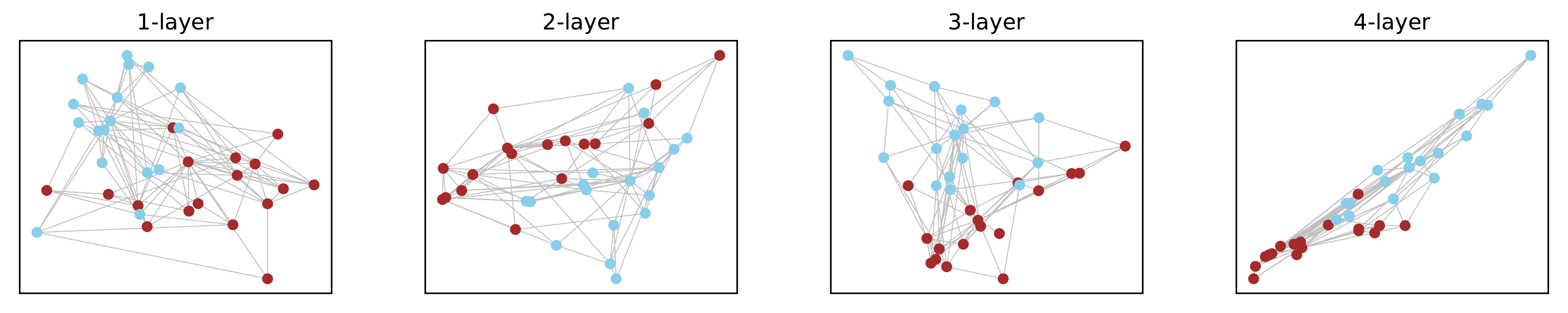}
				\caption{GCN-LPA on the noisy network}
				\label{fig:app_karate_4}
			\end{subfigure}
			\caption{Visualization of GCN and GCN-LPA with 1 $\sim$ 4 layers on karate club network.}
			\label{fig:app_karate}
		\end{figure}

	\subsection{More Visualization Results on Karate Club Network}
	\label{app:e}
		Figure \ref{fig:app_karate} illustrates more visualization of GCN and GCN-LPA on karate club network.
		In each subfigure, we vary the number of layers from 1 to 4 to examine how the learned representations evolve.
		The initial node features are one-hot identity vectors, and the dimension of hidden layers and output layer is 2.
		The transformation matrices are uniformly initialized within range [-1, 1].
		We use sigmoid function as the nonlinear activation function.
		Comparing the four figures in each row, we conclude that the aggregation step and transformation step in GCN and GCN-LPA do benefit the separation of different classes.
		Comparing Figure \ref{fig:app_karate_1} and \ref{fig:app_karate_3} (or Figure \ref{fig:app_karate_2} and \ref{fig:app_karate_4}), we conclude that more inter-class edges will make the separation harder for GCN (or GCN-LPA).		Comparing Figure \ref{fig:app_karate_1} and \ref{fig:app_karate_2} (or Figure \ref{fig:app_karate_3} and \ref{fig:app_karate_4}), we conclude that GCN-LPA is more noise-resistant than GCN, therefore, GCN-LPA can better differentiate classes and identify clustering substructures.

	\subsection{Hyper-parameter Settings}
	\label{app:f}
		The detailed hyper-parameter settings for all datasets are listed in Table \ref{table:ps}.
		In GCN-LPA, we use the same dimension for all hidden layers.
		Note that the number of GCN layers and the number of LPA iterations can actually be different since GCN and LPA are implemented as two independent modules.
		We use grid search to determine hyper-parameters on Cora, and perform fine-tuning on other datasets, i.e., varying one hyper-parameter per time to see if the performance can be further improved.
		The search spaces for hyper-parameters are as follows:
		\begin{itemize}
			\item Dimension of hidden layers: $\{8, 16, 32\}$;
			\item \# GCN layers: $\{1, 2, 3, 4, 5, 6\}$;
			\item \# LPA iterations: $\{1, 2, 3, 4, 5, 6, 7, 8, 9\}$;
			\item L2 weight: $\{ 10^{-7}, 2 \times 10^{-7}, 5 \times 10^{-7}, 10^{-6}, 2 \times 10^{-6}, 5 \times 10^{-6}, 10^{-5}, 2 \times 10^{-5}, 5 \times 10^{-5}, 10^{-4}, 2 \times 10^{-4}, 5 \times 10^{-4}, 10^{-3} \}$;
			\item LPA weight ($\lambda$): $\{0, 1, 2, 5, 10, 15, 20 \}$;
			\item Dropout rate: $\{0, 0.1, 0.2, 0.3, 0.4, 0.5\}$;
			\item Learning rate: $\{0.01, 0.02, 0.05, 0.1, 0.2, 0.5\}$;
		\end{itemize}
		
		\begin{table}[h]
			\centering
			\setlength{\tabcolsep}{10pt}
			\begin{tabular}{c|ccccc}
				\hline
				& \textbf{Cora} & \textbf{Citeseer} & \textbf{Pubmed} & \textbf{Coauthor-CS} & \textbf{Coauthor-Phy} \\
				\hline
				Dimension of hidden layers & 32 & 16 & 32 & 32 & 32 \\
				\# GCN layers & 5 & 2 & 2 & 2 & 2 \\
				\# LPA iterations & 5 & 5 & 1 & 2 & 3 \\
				L2 weight & $1 \times 10^{-4}$ & $5 \times 10^{-4}$ & $2 \times 10^{-4}$ & $1 \times 10^{-4}$ & $1 \times 10^{-4}$ \\
				LPA weight ($\lambda$) & 10 & 1 & 1 & 2 & 1 \\
				Dropout rate & 0.2 & 0 & 0 & 0.2 & 0.2 \\
				Learning rate & 0.05 & 0.2 & 0.1 & 0.1 & 0.05 \\
				\hline
			\end{tabular}
			\caption{Hyper-parameter settings for all datasets.}
			\label{table:ps}
		\end{table}

\end{document}